\documentclass{article}

\usepackage[utf8]{inputenc} 
\usepackage[T1]{fontenc}    
\usepackage{hyperref}       
\usepackage{url}            
\usepackage{booktabs}       
\usepackage{amsfonts}       
\usepackage{microtype}      
\usepackage{xcolor}         
\usepackage{graphicx}

\usepackage{amsmath}

\PassOptionsToPackage{numbers}{natbib}
\usepackage[preprint]{neurips_2026}

\usepackage{natbib}

\hypersetup{
    colorlinks=true,
    citecolor=black
}

\usepackage{algorithm}
\usepackage{algorithmic}

\usepackage{amsthm}
\usepackage{booktabs}
\usepackage{graphicx}
\usepackage[table]{xcolor}
\usepackage{graphicx} 

\newtheorem{theorem}{Theorem}

\newtheorem{lemma}[theorem]{Lemma}

\newtheorem{definition}[theorem]{Definition}
\newtheorem{assumption}[theorem]{Assumption}

\title{Selective Off-Policy Reference Tuning with Plan Guidance}

\author{%
Anh Duc\textsuperscript{1,*}
\hspace{0.45em}
Tien-Phat Nguyen\textsuperscript{1,*}
\hspace{0.45em}
Thien Huu Nguyen\textsuperscript{2}
\hspace{0.45em}
Linh Ngo Van\textsuperscript{1,\ensuremath{\dagger}}
\hspace{0.45em}
Trung Le\textsuperscript{3}
\\[0.6em]
{\normalfont\small \textsuperscript{1}Hanoi University of Science and Technology, Hanoi, Vietnam}
\\
{\normalfont\small \textsuperscript{2}University of Oregon, Eugene, OR, USA}
\\
{\normalfont\small \textsuperscript{3}Monash University, Melbourne, Australia}
}

\begin{document}

{\hypersetup{linkcolor=black}\maketitle}
\begingroup
\renewcommand{\thefootnote}{\fnsymbol{footnote}}
\footnotetext[1]{Equal contribution.}
\footnotetext[2]{Corresponding author: Linh Ngo Van, \texttt{linhnv@soict.hust.edu.vn}.}
\endgroup


\begin{abstract}
Reinforcement learning from verifiable rewards improves reasoning by reinforcing sampled solutions that receive positive outcomes, but group-relative methods such as GRPO become silent on hard prompts where every sampled rollout is wrong. These zero-reward prompts are often the most informative failures: each comes with a verified reference solution, yet uniformly imitating the full trace mixes core reasoning decisions with routine algebra, formatting, and surface wording. We propose \textbf{Selective Off-Policy Reference Tuning with Plan Guidance (SORT)}, an auxiliary repair update that leaves GRPO rollout generation unchanged. For each failed prompt, SORT extracts a reference-derived reasoning plan and scores every reference token twice under the same model, once with only the problem and once with the plan added to the context. Tokens whose probabilities rise under plan conditioning are treated as structurally informative and receive larger Dynamic Fine-Tuning updates, while tokens outside the model's support remain controlled by the base probability. We formalize this mechanism with a ground-truth plan model and prove that the SORT weight approaches an oracle structural-token weight as the extracted plan better approximates the true plan-conditioned distribution. Across three instruction-tuned backbones and eight in-distribution and out-of-distribution reasoning benchmarks, SORT consistently improves over GRPO and guidance-based baselines, with the largest gains on the weakest model where zero-reward failures are most frequent.
\end{abstract}

\section{Introduction}

RLVR has made a simple promise attractive: if correctness can be verified, a language model can improve its reasoning by sampling solutions and reinforcing the successful ones. Group Relative Policy Optimization (GRPO) operationalizes this idea by comparing rollouts for the same prompt, avoiding a learned value model while exploiting verifiable rewards in domains such as mathematics and code \citep{grpo, deepseekai2025}. But this comparison-based signal has a blind spot exactly where capability expansion is needed most: when every sampled solution to a hard prompt is wrong, there is nothing to compare, the group-relative advantage is zero, and the prompt contributes no gradient. In GRPO, the hardest failures can become invisible.

This failure mode is not a corner case of implementation; it is a consequence of sparse outcome rewards and finite sampling. A model can be close enough to benefit from RL on medium-difficulty prompts, yet too weak to sample even one correct trajectory on the prompts that would teach new reasoning patterns. These \emph{zero-reward prompts} are therefore a learning bottleneck: skipping them wastes the most valuable training signal, but standard GRPO cannot learn from them because the reward group is degenerate.

Recent work shows that guidance is necessary when exploration fails, but differs in where that guidance enters. LUFFY injects off-policy traces from stronger models \citep{luffy}; ReLIFT interleaves RL with SFT on hard demonstrations \citep{relift}; Scaf-GRPO changes rollout generation through hierarchical hints \citep{scaf_grpo}; ExPO generates answer-conditioned self-explanations as positive samples \citep{zhou2026expounlockinghardreasoning}; and curriculum or resampling methods reallocate compute away from degenerate groups \citep{xiong2025minimalistapproachllmreasoning, dapo, li2025knapsackrlunlockingexploration, chen2026nudgingboundariesllmreasoning, kimiteam2026kimik2openagentic}. These methods all help restore a learning signal, but they do so by changing the training distribution, relying on external or synthesized trajectories, or applying relatively uniform imitation to complete demonstrations.

This leaves a narrower gap than simply ``hard prompts need supervision.'' Once GRPO has produced an zero-reward group, the reference solution is useful, but using it well requires deciding \emph{which parts} of the reference should shape the student. A full solution contains core reasoning decisions mixed with surface wording, formatting, and routine algebra. Reference-level objectives can move the model toward the verified path, but they do not by themselves identify which tokens express the reasoning structure that the student failed to discover. The missing ingredient is therefore a hard-question repair signal that is simultaneously reference-supported, stable, and token-selective, while preserving the original GRPO rollout process. We target this gap by treating the reference solution not as a trajectory to copy, but as evidence for a reasoning plan from which to decide \emph{what} to learn: a reference-derived plan should reveal which reference tokens become predictable only when the underlying reasoning plan is made explicit, thereby turning a verified solution into structure-guided token-level credit.

We propose \textbf{Selective Off-Policy Reference Tuning with Plan Guidance (SORT)}, an auxiliary repair mechanism for zero-reward GRPO groups. SORT leaves the main rollout process unchanged. When a prompt produces only incorrect rollouts, we buffer its verified reference solution and extract a reference-derived reasoning plan. We then evaluate the same model twice along the reference path: once under the original prompt, and once under a plan-conditioned teacher prompt. The plan is not used for rollout or inference, but serves as a diagnostic context. Tokens whose likelihood increases significantly under plan conditioning are interpreted as carrying structural reasoning information. This teacher--student confidence gap converts the reference trajectory into a token-selective learning signal instead of a uniformly imitated sequence.

We turn this diagnostic gap into a conservative auxiliary objective: plan-conditioned relevance decides where the reference solution is informative, while a DFT-inspired probability-space update controls how strongly the model is moved toward those tokens \cite{wu2026generalizationsftreinforcementlearning}. Specifically, each reference token is weighted by the geometric mean of the student probability and the plan-conditioned teacher probability. This form preserves the stability of reference-path learning while amplifying updates on tokens supported by the inferred reasoning structure. We formalize the approach through a ground-truth plan model, where structural tokens are defined via probability gains under plan conditioning, and show that the resulting update converges to an oracle that emphasizes precisely such tokens, with approximation error governed by plan quality. As a result, SORT repairs zero-reward prompts without modifying rollout distributions, introducing external trajectories, or collapsing to standard SFT.
\paragraph{Contributions.}
We make the following contributions:
\begin{itemize}
    \item We identify zero-reward prompts as GRPO zero-advantage failures and introduce SORT, which repairs them without changing rollout generation.
    
    \item We use same-model plan conditioning only to score reference tokens, avoiding external teachers, rollout hints, or off-policy trajectory injection.
    
    \item We propose a geometric-mean token weight that combines DFT-style support awareness with plan-conditioned selectivity.
    
    \item We provide theory and experiments showing that SORT approximates oracle structural-token weighting and improves over GRPO and guidance-based baselines.
\end{itemize}

\section{Related Work}

\paragraph{RLVR and hard-question failure.}
RLVR improves reasoning with verifiable rewards, with GRPO as a common group-relative objective \citep{grpo, deepseekai2025}. Its key failure mode is the all-wrong or zero-variance group, where relative advantages collapse to zero. Prior work addresses this by shaping zero-variance advantages \citep{le2026nopromptleftbehind}, reallocating rollout budgets \citep{li2025knapsackrlunlockingexploration, nguyen2026adaptive_rollout}, or filtering/resampling near the model's competence band \citep{xiong2025minimalistapproachllmreasoning, dapo, chen2026nudgingboundariesllmreasoning, kimiteam2026kimik2openagentic}. SORT instead keeps all-wrong prompts and converts their verified references into auxiliary token-level learning signals.

\paragraph{Guidance-based repair for hard reasoning.}
Guidance methods restore signal by changing the training context or trajectories: LUFFY mixes stronger off-policy traces \citep{luffy}; SAGE, Guide, and Scaf-GRPO inject hints or scaffolds \citep{sage, nath2025adaptive_guidance, scaf_grpo}; SGPO and ExPO add stepwise guidance or answer-conditioned explanations \citep{chen2026stepwiseguidedpolicyoptimization, zhou2026expounlockinghardreasoning}; and ReLIFT interleaves supervised hard-example updates \citep{relift}. Related structured-repair methods use targeted exploration, online self-verification, or decomposed rubric rewards \citep{zheng2025fr3e, liu2025rise, gangwani2026rubric_reward}. SORT differs by leaving rollout generation unchanged and using verified references only for structure-guided, token-selective repair.

\paragraph{Reference learning and token-selective supervision.}
SORT also relates to reference-learning and process-supervision objectives. DFT reweights SFT in probability space \citep{wu2026generalizationsftreinforcementlearning}, SDFT uses a demonstration-conditioned same-model teacher \citep{shenfeld2026selfdistillationenablescontinuallearning}, and RLPR uses reference-answer likelihoods as verifier-free rewards \citep{yu2025rlpr}. Process-supervision methods assign value to partial reasoning states via step, segment, or implicit process signals \citep{zhang2025lessons_prm, dai2025sgrpo, yuan2024free_process_rewards, cui2025prime}; SORT instead avoids a separate process verifier or forced early exits, using same-model plan conditioning to upweight structurally informative reference tokens. This aligns with token-level analyses showing that useful RL signal concentrates on sparse decision-critical tokens \citep{wang2025beyond8020rule, vassoyan2025ignoreklpenalty}.

    \section{Methodology}
    \label{sec:method}

    We introduce \textbf{Selective Off-Policy Reference Tuning with Plan Guidance (SORT)}, an auxiliary reference update that converts zero-reward GRPO prompts into token-selective learning signals while leaving rollout generation unchanged.

    \subsection{From Zero-Reward Failures to Deferred Repair}

    Verified references are valuable on hard prompts, but copying them wholesale is the wrong repair mechanism: cross-entropy SFT can narrow the policy and full reasoning traces mix decision-critical steps with routine algebra and wording \citep{li2025preservingdiversitysft,chen2026sedsft}. GRPO avoids this uniform off-policy imitation by learning from sampled rollouts, yet it becomes silent when every rollout is wrong.

    Consider a GRPO training iteration with rollout batch $\mathcal{U}_k = \{q_i\}_{i=1}^{B_{\mathrm{rl}}}$. For each prompt $q_i$, the model samples $N$ responses and receives binary verifiable rewards $r_i^{(n)}\in\{0,1\}$. GRPO assigns each sampled response a normalized group advantage,
    \begin{equation}
    \bar r_i = \frac{1}{N}\sum_{m=1}^{N} r_i^{(m)},\qquad
    s_i = \sqrt{\frac{1}{N}\sum_{m=1}^{N}(r_i^{(m)}-\bar r_i)^2},\qquad
    A_i^{(n)} = \frac{r_i^{(n)}-\bar r_i}{s_i+\epsilon}.
    \end{equation}
    If every rollout is incorrect, then
    \begin{equation}
    r_i^{(1)}=\cdots=r_i^{(N)}=0
    \;\Longrightarrow\;
    \bar r_i=s_i=0,\quad A_i^{(n)}=0\ \forall n,
    \end{equation}
    so all policy-gradient terms for $q_i$ vanish. We denote these zero-reward indices by $\mathcal{I}_{\mathrm{aw}}^{(k)}=\{i:r_i^{(n)}=0,\forall n\in[N]\}$. They are not noise---they mark the policy frontier, where a single missed reasoning decision can make an otherwise fluent rollout fail. The verified reference $y_i^\star$ contains the missing structure, but interleaved with routine derivation. The bottleneck is \emph{selectivity}: learn from the tokens that reveal the missing structure without collapsing into uniform imitation~\citep{wu2026generalizationsftreinforcementlearning}. Appendix~\ref{app:hard_examples} gives representative zero-reward prompts.

    SORT resolves this by deferring repair. Failed prompts are neither skipped nor merged into GRPO; they are \textbf{buffered} for a separate auxiliary update:
    \begin{equation}
    \mathcal{B}_{\mathrm{aw}} \leftarrow \operatorname{Enqueue}\!\bigl(\mathcal{B}_{\mathrm{aw}},\; \{(q_i, y_i^\star) : i \in \mathcal{I}_{\mathrm{aw}}^{(k)}\}\bigr).
    \label{eq:enqueue}
    \end{equation}
    When $|\mathcal{B}_{\mathrm{aw}}| \ge B_{\mathrm{aux}}$, a minibatch $\mathcal{M}_k \sim \operatorname{SampleBatch}(\mathcal{B}_{\mathrm{aw}}, B_{\mathrm{aux}})$ triggers one SORT update. GRPO continues learning from prompts with reward variance; SORT repairs the hard frontier.

    \subsection{Plan-Conditioned Token Salience}
    \label{sec:plan_conditioned_salience}

    For a buffered pair $(q_i, y_i^\star)$, SORT asks which reference tokens should exert the strongest learning pressure. The guiding principle is that a token should be emphasized when it becomes substantially more predictable once the model is given the ground-truth reasoning plan behind the verified solution.

    \begin{definition}[Ground-truth solution plan]
    \label{def:ground_truth_solution_plan}
    For a buffered example $(q_i,y_i^\star)$ with
    $y_i^\star=(y_{i,1}^\star,\ldots,y_{i,T_i}^\star)$, let
    $S_i$ denote the \emph{ground-truth solution plan}: an idealized ordered outline of the key reasoning decisions sufficient to derive the reference solution from the prompt. It records structural choices such as representations, lemmas, case splits, and constraint propagation. We do not assume that $S_i$ is observed, optimized as a target, or sampled by the model; it is used only as an analysis device. Conditioned on a fixed plan $S_i$, the reference factorizes as
    \begin{equation}
    \pi_\theta(y_i^\star \mid q_i, S_i)
    =
    \prod_{t=1}^{T_i}
    \pi_\theta(y_{i,t}^\star \mid q_i, S_i, y_{i,<t}^\star).
    \label{eq:gt_plan_factorization}
    \end{equation}
    \end{definition}

    Appendix~\ref{app:ground_truth_plan_example} illustrates the idealized ground-truth plan used in the analysis.

    \begin{definition}[Plan salience]
    \label{def:plan_salience}
    The ground-truth plan salience of reference token $y_{i,t}^\star$ is
    \begin{equation}
    \sigma_{i,t} := \log\frac{\pi_\theta(y_{i,t}^\star \mid q_i, S_i, y_{i,<t}^\star)}{\pi_\theta(y_{i,t}^\star \mid q_i, y_{i,<t}^\star)}.
    \label{eq:plan_salience}
    \end{equation}
    Equivalently, this log-ratio is the conditional pointwise mutual information $\operatorname{PMI}_\theta(y_{i,t}^\star; S_i \mid q_i, y_{i,<t}^\star)$, which clarifies its role. Large positive values ($\sigma_{i,t} \gg 0$) indicate that the plan contributes substantial information: the base model cannot reliably predict the token from the problem and prefix alone, but becomes confident once the step-by-step plan is revealed. These correspond to reasoning-critical steps—e.g., invariant selection, case splits, or constraint propagation—that determine solution correctness. In contrast, tokens with $\sigma_{i,t} \approx 0$ are routine (algebraic manipulation or phrasing) and largely independent of the plan. SORT therefore ignores the near-zero regime and selectively upweights tokens with large positive $\sigma_{i,t}$.    \end{definition}
        
    If the ground-truth plan $S_i$ were available, $\sigma_{i,t}$ would identify reference tokens that are difficult yet learnable given the correct reasoning steps. Since only the reference solution is available, SORT approximates this conditioning with a reference-derived plan, scoring each token twice under the same model: once with only the problem and once with the extracted plan in context.
    \paragraph{Reference-derived plan.}
    For each buffered example, we ask the current model to extract a plan-like conditioning context from the verified reference:
    \begin{equation}
    \hat{s}_i
    =
    \operatorname{Gen}_\theta
    \!\left(
    \textsc{PlanExtract}(y_i^\star)
    \right).
    \label{eq:plan_extract}
    \end{equation}
    The extracted plan $\hat{s}_i$ is not used as a target sequence and is never used at inference time. It serves only as a diagnostic context for scoring the reference. Its role is to make explicit the decisions, intermediate realizations, and calculation context that support the verified solution, so that plan-dependent reference tokens become easier to predict under conditioning. Prompt templates are given in Appendix~\ref{app:prompt_templates}.
    
\paragraph{Two reference-path evaluations.}
We evaluate the same reference tokens under two prefixes: the base prefix contains only the problem, while the plan-conditioned teacher prefix additionally contains $\hat{s}_i$:
\begin{equation}
\begin{aligned}
p^{\mathrm{base}}_{i,t}
&=
\pi_\theta\!\left(
y_{i,t}^\star
\mid
\textsc{BasePrompt}(q_i), y_{i,<t}^\star
\right),
&
p^{\mathrm{plan}}_{i,t}
&=
\pi_\theta\!\left(
y_{i,t}^\star
\mid
\textsc{TeacherPrompt}(q_i,\hat{s}_i), y_{i,<t}^\star
\right).
\end{aligned}
\label{eq:base_plan_prob}
\end{equation}
Both evaluations use the same model and reference path; only the conditioning context differs. In the auxiliary loss, the base branch receives gradient, while the plan-conditioned branch is detached.

    \paragraph{Plan-conditioned log-ratio.}
    Using the extracted plan $\hat{s}_i$ as a practical proxy for the ground-truth plan $S_i$ in the log-ratio gives an estimated token-level score:
    \begin{equation}
    \hat{\sigma}_{i,t}
    :=
    \log p^{\mathrm{plan}}_{i,t}
    -
    \log p^{\mathrm{base}}_{i,t}.
    \label{eq:plan_log_ratio}
    \end{equation}
    When the extracted plan induces conditioning close to the ground-truth plan ($\hat{s}_i \approx S_i$), this plug-in estimator approximates the ideal plan salience: $\hat{\sigma}_{i,t} \approx \sigma_{i,t}$. If $\hat{\sigma}_{i,t}$ is large and positive, the reference token becomes much easier to predict once the plan is supplied, indicating it encodes a reasoning step the base policy is missing. SORT interprets such tokens as more informative for updating the base policy and upweights them accordingly.

    \subsection{Geometric Reference Weighting}
    \label{sec:geometric_reference_weighting}

    We use lowercase $\ell_{i,t}$ for a single-token loss and uppercase $\mathcal{L}$ for an objective aggregated over a token set or minibatch.
    Dynamic Fine-Tuning (DFT) \citep{wu2026generalizationsftreinforcementlearning} stabilizes reference-path learning by replacing the unit SFT weight with the detached base probability,
    \[
    \ell^{\mathrm{DFT}}_{i,t}
    =
    -\operatorname{sg}\!\left(p^{\mathrm{base}}_{i,t}\right)
    \log p^{\mathrm{base}}_{i,t}.
    \]
    Its gradient can be written in policy-gradient form as
    \begin{equation}
    \nabla_\theta \ell^{\mathrm{DFT}}_{i,t}
    =
    -\nabla_\theta p^{\mathrm{base}}_{i,t}
    =
    -\mathbb{E}_{z\sim \pi_\theta(\cdot\mid q_i, y_{i,<t}^\star)}
    \!\left[
    \underbrace{\mathbf{1}\{z=y_{i,t}^\star\}}_{\text{reward}=1}\,
    \nabla_\theta \log \pi_\theta(z\mid q_i, y_{i,<t}^\star)
    \right].
    \label{eq:dft_as_rl}
    \end{equation}
    Every reference token receives a uniform reward of~1, which stabilizes training but ignores the distinction between reasoning-critical and routine tokens. SORT generalizes this by replacing the uniform reward with a plan-conditioned one.

    \paragraph{From uniform reward to plan-conditioned reward.}
    By Definition~\ref{def:plan_salience}, the plan salience $\sigma_{i,t}=\log\bigl(\pi_\theta(y_{i,t}^\star\mid q_i,S_i,y_{i,<t}^\star)\,/\,p^{\mathrm{base}}_{i,t}\bigr)$ measures how much the oracle plan $S_i$ increases the token probability. We replace DFT's uniform reward with the tempered plan-conditioned reward:
    \begin{equation}
    \rho^\star_{\beta,i,t}
    =
    \exp(\beta \sigma_{i,t}),
    \qquad
    \beta \in [0,1].
    \label{eq:oracle_structural_reward}
    \end{equation}
    Setting $\beta=0$ recovers DFT; any $\beta>0$ upweights tokens that become more predictable once the plan is revealed ($\rho^\star_{\beta,i,t}>1$ iff $\sigma_{i,t}>0$). Inserting this reward into the DFT gradient gives
    \begin{align}
    g^\star_{\beta,i,t}
    &=
    -\rho^\star_{\beta,i,t}\,\nabla_\theta p^{\mathrm{base}}_{i,t}
    =
    -\mathbb{E}_{z\sim \pi_\theta(\cdot\mid q_i, y_{i,<t}^\star)}
    \!\left[
    \mathbf{1}\{z=y_{i,t}^\star\}\,
    \rho^\star_{\beta,i,t}\,
    \nabla_\theta \log \pi_\theta(z\mid q_i, y_{i,<t}^\star)
    \right]
    \nonumber\\
    &=
    \nabla_\theta
    \left[
    -\operatorname{sg}
    \!\left(
    p^{\mathrm{base}}_{i,t}\,\rho^\star_{\beta,i,t}
    \right)
    \log p^{\mathrm{base}}_{i,t}
    \right],
    \label{eq:oracle_structural_gradient}
    \end{align}
    Reading off the surrogate loss, the oracle token weight and the corresponding oracle objective are
    \begin{align}
    \omega^\star_{\beta,i,t}
    &=
    \operatorname{sg}
    \!\left(
    p^{\mathrm{base}}_{i,t}\,
    \exp(\beta \sigma_{i,t})
    \right),
    \label{eq:oracle_weight}
    \\
    \mathcal{L}^\star(\mathcal{M}_k)
    &=
    -\frac{1}{|\mathcal{S}(\mathcal{M}_k)|}
    \sum_{(i,t)\in \mathcal{S}(\mathcal{M}_k)}
    \omega^\star_{\beta,i,t}\,
    \log\pi_\theta
    \!\left(
    y_{i,t}^\star \mid \textsc{BasePrompt}(q_i), y_{i,<t}^\star
    \right).
    \label{eq:oracle_loss}
    \end{align}
    If the ground-truth plan $S_i$ were known, $\mathcal{L}^\star$ would be the ideal training objective. In practice $\sigma_{i,t}$ is unavailable, so we approximate it.

    \paragraph{From oracle to practical weight.}
    Replacing $\sigma_{i,t}$ with the estimated log-ratio $\hat{\sigma}_{i,t}$ from Eq.~\eqref{eq:plan_log_ratio} gives the practical SORT weight family:
    \begin{equation}
    \omega_{\beta,i,t}
    =
    \operatorname{sg}
    \!\left(
    p^{\mathrm{base}}_{i,t}\,
    \exp(\beta\,\hat{\sigma}_{i,t})
    \right)
    =
    \operatorname{sg}
    \!\left(
    \left(p^{\mathrm{base}}_{i,t}\right)^{1-\beta}
    \left(p^{\mathrm{plan}}_{i,t}\right)^{\beta}
    \right),
    \qquad \beta \in [0,1].
    \label{eq:sort_weight_family}
    \end{equation}
    At $\beta=0$ the weight reduces to $\operatorname{sg}(p^{\mathrm{base}}_{i,t})$, recovering DFT. As $\beta$ increases, the plan-conditioned probability $p^{\mathrm{plan}}_{i,t}$ gains influence: tokens whose probability rises under plan conditioning are upweighted, while those unaffected or hurt are unchanged or downweighted. We set $\beta=\tfrac{1}{2}$ throughout, which yields the geometric-mean weight and the minibatch objective:
    \begin{align}
    \omega_{i,t}
    &=
    \operatorname{sg}
    \!\left(
    \sqrt{p^{\mathrm{base}}_{i,t}\,p^{\mathrm{plan}}_{i,t}}
    \right),
    \label{eq:sort_weight}
    \\
    \mathcal{L}_{\mathrm{SORT}}(\mathcal{M}_k)
    &=
    -\frac{1}{|\mathcal{S}(\mathcal{M}_k)|}
    \sum_{(i,t)\in \mathcal{S}(\mathcal{M}_k)}
    \omega_{i,t}\,
    \log\pi_\theta
    \!\left(
    y_{i,t}^\star \mid \textsc{BasePrompt}(q_i), y_{i,<t}^\star
    \right),
    \label{eq:sort_loss}
    \end{align}
    where $\mathcal{S}(\mathcal{M}_k) = \{(i,t) : (q_i,y_i^\star)\in \mathcal{M}_k,\; 1 \le t \le |y_i^\star|\}$ and only the base branch receives gradient. Tokens that become more predictable under plan conditioning ($p^{\mathrm{plan}}_{i,t} > p^{\mathrm{base}}_{i,t}$) are upweighted relative to DFT; tokens made less predictable are downweighted.
    \subsection{Theoretical Analysis}
    \label{sec:theoretical_analysis}

    The weighting scheme above relies on the extracted plan $\hat{s}_i$ in place of the unknown ground-truth plan $S_i$. The plan mismatch $\gamma$ below should be small when the policy has reasonable planning ability: as reasoning improves, extracted plans should better preserve high-level solution structure, as illustrated qualitatively in Appendix~\ref{app:plan_extraction_example}. Thus the bound is small under faithful extraction and weakens when the extractor misses key structural decisions.

    \begin{assumption}[Plan-induced stability]
    \label{assump:stability}
    Let $h_\theta(q, y_{<t}^\star, s)$ denote the decoder hidden state at step $t$ conditioned on plan $s$. We assume that for any $\hat{s}_i, S_i$ and all token positions $t$,
    \[
    \| h_\theta(q_i, y_{i,<t}^\star, \hat{s}_i) - h_\theta(q_i, y_{i,<t}^\star, S_i) \| \le \gamma,
    \]
    for some $\gamma \ge 0$.
    \end{assumption}

    \begin{assumption}[Bounded score]
    \label{assump:bounded_score}
    The token-level score function along the reference path is uniformly bounded:
    \[
    \bigl\|\nabla_\theta \log \pi_\theta(y_{i,t}^\star \mid q_i, y_{i,<t}^\star)\bigr\| \le G
    \qquad \forall (i,t).
    \]
    \end{assumption}

    \begin{lemma}[Plan-to-token smoothness]
    \label{lem:plan_to_token}
    Assume the output mapping $h \mapsto \log \pi_\theta(\cdot \mid h)$ is $L$-Lipschitz. Under Assumption~\ref{assump:stability}, we have
    \[
    \bigl|\log\pi_\theta(y_{i,t}^\star \mid \hat{s}_i, q_i, y_{i,<t}^\star)
    -
    \log\pi_\theta(y_{i,t}^\star \mid S_i, q_i, y_{i,<t}^\star)\bigr|
    \le L \gamma.
    \]
    \end{lemma}

    Applying Lemma~\ref{lem:plan_to_token} and the mean value theorem yields the following regret bound.

    \begin{theorem}[Plan-regret bound]
    \label{thm:regret_bound}
    Under Assumption~\ref{assump:stability}, for any $\beta \in [0,1]$, the per-token SORT weight error satisfies
    \begin{equation}
    \boxed{\;
    \bigl|\omega_{\beta,i,t} - \omega_{\beta,i,t}^\star\bigr|
    \;\le\;
    \beta L\gamma \left(p^{\mathrm{base}}_{i,t}\right)^{1-\beta},
    \;}
    \end{equation}
    where $\omega^\star_{\beta,i,t}$ is the oracle weight from Eq.~\eqref{eq:oracle_weight}. In particular, for the geometric-mean setting $\beta=\tfrac{1}{2}$,
    \[
    \bigl|\omega_{i,t} - \omega_{i,t}^\star\bigr|
    \le
    \frac{L\gamma}{2} \sqrt{p^{\mathrm{base}}_{i,t}}.
    \]
    Under Assumption~\ref{assump:bounded_score}, aggregated over the auxiliary minibatch,
    \begin{equation}
    \boxed{\;
    \bigl\|
    \nabla_\theta \mathcal{L}_{\mathrm{SORT},\beta} - \nabla_\theta \mathcal{L}^\star_\beta
    \bigr\|
    \;\le\;
    \beta L\gamma \cdot \frac{1}{|\mathcal{S}|} \sum_{(i,t)\in\mathcal{S}} \left(p^{\mathrm{base}}_{i,t}\right)^{1-\beta}\; \bigl\| \nabla_\theta \log p^{\mathrm{base}}_{i,t} \bigr\|
    \;\le\;
    \beta L\gamma G.
    \;}
    \end{equation}
    \end{theorem}

    The bound carries three messages. First, as $\gamma \to 0$ the SORT gradient converges to the oracle gradient at rate $\mathcal{O}(\gamma)$. Second, the factor $\left(p^{\mathrm{base}}_{i,t}\right)^{1-\beta}$ automatically suppresses the weight error on tokens the model finds difficult ($p^{\mathrm{base}}_{i,t} \ll 1$), inheriting DFT's stabilization; for the midpoint choice $\beta=\tfrac{1}{2}$ this becomes $\sqrt{p^{\mathrm{base}}_{i,t}}$. Third, under bounded scores, the aggregated gradient deviation is at most $\beta L\gamma G$, independent of sequence length or vocabulary size. In short, SORT degrades gracefully: it exploits structural information when the plan is faithful, and falls back to DFT-like behavior otherwise. Proofs are deferred to Appendix~\ref{app:regret_bound}.
    \subsection{Interleaved Optimization}
    \label{sec:interleaved_optimization}

    SORT is applied as an auxiliary step interleaved with standard GRPO: each iteration performs a GRPO update on the rollout batch, zero-reward prompts are buffered, and when the buffer reaches size $B_{\mathrm{aux}}$, a minibatch triggers one SORT update. This decouples on-policy exploration from reference consolidation. The full procedure is presented in Algorithm~\ref{alg:sort}.

\section{Experiments}
\label{sec:experiments}

\subsection{Experimental Setup}

\paragraph{Models, data, and evaluation.}
We assess SORT on Llama-3.2-3B-Instruct \citep{llama3}, Qwen2.5-7B-Instruct \citep{qwen25technicalreport}, and Qwen3-4B-Instruct \citep{qwen3technicalreport}. Training uses a 15k-prompt subset of OpenR1-Math-220k, curated by Liao et al. \citep{sage}, derived from NuminaMath 1.5 \citep{li2024numinamath}, and paired with ground-truth answers and reference solutions, without pass-rate filtering. We report in-distribution performance on AIME24/25 \citep{aime2024,aime2025}, AMC23 \citep{amc23}, MATH-500 \citep{hendrycks2021measuring}, Minerva Math \citep{minerva2022}, and OlympiadBench \citep{olympiadbench}, and out-of-distribution performance on GPQA-diamond \citep{gpqa} and MMLU-Pro \citep{mmlupro}. Unless otherwise noted, all metrics are avg@16 accuracy.

\paragraph{Baselines and implementation.}
We compare SORT with SFT on DeepSeek-R1 traces, vanilla GRPO, LUFFY, ReLIFT, and Scaf-GRPO. All guidance baselines are reproduced on the same 15k training subset with matched batch size and optimization steps. SORT uses only same-model plans to score reference tokens and does not modify rollout generation. All methods are trained on 8 H100 GPUs using \texttt{verl} \citep{Sheng_2025} and vLLM \citep{kwon2023efficientmemorymanagementlarge}; following DAPO \citep{dapo}, we disable KL regularization and use asymmetric clipping with \((\epsilon_{\mathrm{low}},\epsilon_{\mathrm{high}})=(0.2,0.28)\). Unless otherwise specified, we use a learning rate of \(1\times10^{-6}\), maximum response length 8192, batch size 128, 8 trajectories per prompt, and 500 training steps, we save every 50 steps. For Qwen3-4B-Instruct, we use 4 trajectories per prompt  due to slower training caused by its long response length.

\paragraph{Buffer size.}
We set the SORT buffer size from the early all-wrong count in Appendix~\ref{app:all_wrong_fraction}, using the largest power of two below that count so the buffer fills quickly. This gives buffer sizes of \(64\), \(16\), and \(8\) for Llama-3.2-3B-Instruct, Qwen2.5-7B-Instruct, and Qwen3-4B-Instruct, respectively.
\begin{table*}[t]
\centering
\caption{Main results on in-distribution math benchmarks and out-of-distribution reasoning benchmarks. All numbers are avg@16 accuracy. In the AIME24/25 column, the two values correspond to AIME24 and AIME25, respectively. $\Delta$ is measured against the corresponding base model.}
\label{tab:main_results}
\setlength{\tabcolsep}{4.5pt}
\renewcommand{\arraystretch}{1.08}
\definecolor{oursrow}{RGB}{238,243,250}

\small
\resizebox{\textwidth}{!}{%
\begin{tabular}{lccccccccccc}
\toprule
& \multicolumn{7}{c}{\textbf{In-distribution}}
& \multicolumn{4}{c}{\textbf{Out-of-distribution}} \\
\cmidrule(lr){2-8} \cmidrule(lr){9-12}
\textbf{Method}
& \textbf{AIME24/25} & \textbf{AMC23} & \textbf{MATH-500} & \textbf{Minerva} & \textbf{Olympiad} & \textbf{Avg.} & $\Delta$
& \textbf{GPQA} & \textbf{MMLU-Pro} & \textbf{Avg.} & $\Delta$ \\
\midrule

\textit{Llama-3.2-3B-Instruct}
& 6.5/0.6 & 22.8 & 44.7 & 17.8 & 14.2 & 17.8 & 0
& 17.9 & 27.0 & 22.5 & 0 \\
SFT
& 0.4/0.6 & 9.5  & 26.9 & 5.1  & 6.5  & 8.2  & -9.6
& 11.6 & 18.8 & 15.2 & -7.3 \\
GRPO
& 6.7/0.8 & \underline{29.5} & \underline{52.1} & \underline{20.5} & \underline{21.8} & \underline{21.9} & \underline{+4.1}
& \underline{26.3} & \underline{39.8} & \underline{33.1} & \underline{+10.6} \\
LUFFY
& 4.4/0.4 & 18.6 & 38.9 & 14.3 & 11.9 & 14.7 & -3.1
& 16.0 & 26.7 & 21.4 & -1.1 \\
ReLIFT
& 0.4/0.2 & 17.8 & 36.8 & 10.4 & 10.0 & 12.6 & -5.2
& 11.7 & 26.5 & 19.1 & -3.4 \\
Scaf-GRPO
& \underline{7.7}/\textbf{2.3} & 28.8 & 51.7 & 19.4 & 19.5 & 21.5 & +3.7
& 24.1 & 38.0 & 31.0 & +8.5 \\
\rowcolor{oursrow}
\textbf{Ours}
& \textbf{8.8}/\underline{1.7} & \textbf{32.7} & \textbf{56.0} & \textbf{21.6} & \textbf{22.4} & \textbf{23.9} & \textbf{+6.1}
& \textbf{26.4} & \textbf{41.2} & \textbf{33.8} & \textbf{+11.3} \\
\midrule

\textit{Qwen2.5-7B-Instruct}
& 13.8/6.7 & 53.4 & 75.7 & 38.1 & 39.2 & 37.8 & 0
& 37.1 & 56.4 & 46.7 & 0 \\
SFT
& 3.5/7.1 & 30.9 & 56.2 & 20.0 & 21.7 & 23.2 & -14.6
& 9.5  & 35.6 & 22.5 & -24.2 \\
GRPO
& 15.0/13.5 & 55.5 & 79.2 & 39.1 & \underline{44.5} & 41.1 & +3.3
& 37.2 & 57.6 & 47.4 & +0.7 \\
LUFFY
& \textbf{17.1}/13.5 & 55.2 & \textbf{81.3} & 39.0 & 44.2 & 41.7 & +3.9
& \underline{38.1} & \underline{59.1} & \underline{48.6} & \underline{+1.9} \\
ReLIFT
& \underline{15.6}/\textbf{15.8} & 55.8 & 80.5 & 38.8 & 44.1 & \underline{41.8} & \underline{+4.0}
& 32.4 & 57.3 & 44.9 & -1.8 \\
Scaf-GRPO
& 14.6/12.7 & \underline{58.8} & 78.0 & \textbf{39.8} & 42.0 & 41.0 & +2.2
& 36.6 & 58.4 & 47.5 & +0.8 \\
\rowcolor{oursrow}
\textbf{Ours}
& \textbf{17.1}/\underline{15.4} & \textbf{60.5} & \underline{81.0} & \underline{39.5} & \textbf{47.3} & \textbf{43.5} & \textbf{+5.7}
& \textbf{39.6} & \textbf{60.5} & \textbf{50.1} & \textbf{+3.4} \\
\midrule

\textit{Qwen3-4B-Instruct}
& 52.1/43.1 & 92.2 & 93.6 & 46.1 & 67.7 & 65.8 & 0
& \textbf{57.6} & 70.9 & 64.3 & 0 \\
SFT
& 14.4/22.1 & 55.2 & 78.9 & 38.1 & 40.2 & 41.5 & -24.3
& 29.4 & 52.6 & 41.0 & -23.3 \\
GRPO
& 55.8/46.5 & \textbf{95.0} & \underline{95.6} & \textbf{49.9} & \underline{70.3} & 68.8 & +3.0
& \underline{57.0} & \underline{72.0} & \underline{64.5} & \underline{+0.2} \\
LUFFY
& 50.8/38.1 & 88.8 & 93.3 & 48.1 & 59.4 & 63.1 & -2.7
& 56.0 & 46.6 & 51.3 & -13.0 \\
ReLIFT
& \underline{59.0}/\textbf{48.3} & 93.3 & 95.3 & 48.1 & 69.1 & 68.9 & +3.1
& 56.8 & 71.8 & 64.3 & 0.0 \\
Scaf-GRPO
& 58.3/\textbf{48.3} & \underline{94.2} & 95.4 & \underline{48.8} & 69.8 & \underline{69.1} & \underline{+3.3}
& 53.6 & \textbf{72.1} & 62.9 & -1.4 \\
\rowcolor{oursrow}
\textbf{Ours}
& \textbf{62.5}/\underline{47.5} & 93.4 & \textbf{95.8} & 48.3 & \textbf{71.5} & \textbf{69.8} & \textbf{+4.0}
& \textbf{57.6} & \textbf{72.1} & \textbf{64.9} & \textbf{+0.6} \\
\bottomrule
\end{tabular}%
}
\end{table*}

\subsection{Main Results}
\vspace{-0.5em}
Table~\ref{tab:main_results} shows that SORT is the strongest method on average across all three backbones, covering weak, medium, and strong model regimes. Relative to vanilla GRPO, SORT improves the in-distribution average by 2.0 points on Llama-3.2-3B, 2.4 points on Qwen2.5-7B, and 2.0 points on Qwen3-4B. These gains also transfer out of distribution, where SORT improves the GPQA/MMLU-Pro average by 0.7, 2.7, and 1.4 points, respectively. This pattern matches the intended role of SORT: when GRPO fails to produce a correct rollout on difficult prompts, the auxiliary reference-path update recovers additional learning signal without changing the rollout distribution.

The largest improvement over the untuned base model appears on the weakest backbone. On Llama-3.2-3B-Instruct, SORT increases the in-distribution average from 17.8 to 23.9 and the out-of-distribution average from 22.5 to 33.8, which is consistent with the expectation that all-wrong groups occur most frequently in this regime. On the medium Qwen2.5-7B-Instruct backbone, SORT also gives the strongest in-distribution and out-of-distribution averages, showing that the benefit is not limited to low-capability models. At the same time, SORT still improves the strongest starting point, Qwen3-4B-Instruct, from 65.8 to 69.8 in distribution and from 64.3 to 64.9 out of distribution, indicating that the method remains complementary even when the base policy already solves many prompts. By contrast, LUFFY applies reference supervision much more broadly, effectively mixing ground-truth traces into the update for every prompt and boosting low-probability tokens even when they are not the right learning signal; this makes it less selective than SORT. ReLIFT and pure SFT expose a different failure mode on weak backbones: they tend to make answers longer and closer to reference traces, but this extra verbosity does not translate into better reasoning accuracy. This is most visible on Llama-3.2-3B-Instruct, where SFT and ReLIFT both underperform GRPO despite producing more imitation-like responses. This supports the core design principle of SORT: repairing hard questions should be reference-supported, but the supervision must remain token-selective and policy-aware rather than uniformly imitative.

\subsection{Training Dynamics}
\label{sec:training_dynamics}
\begin{figure*}[t]
    \centering
    \includegraphics[width=0.9\textwidth]{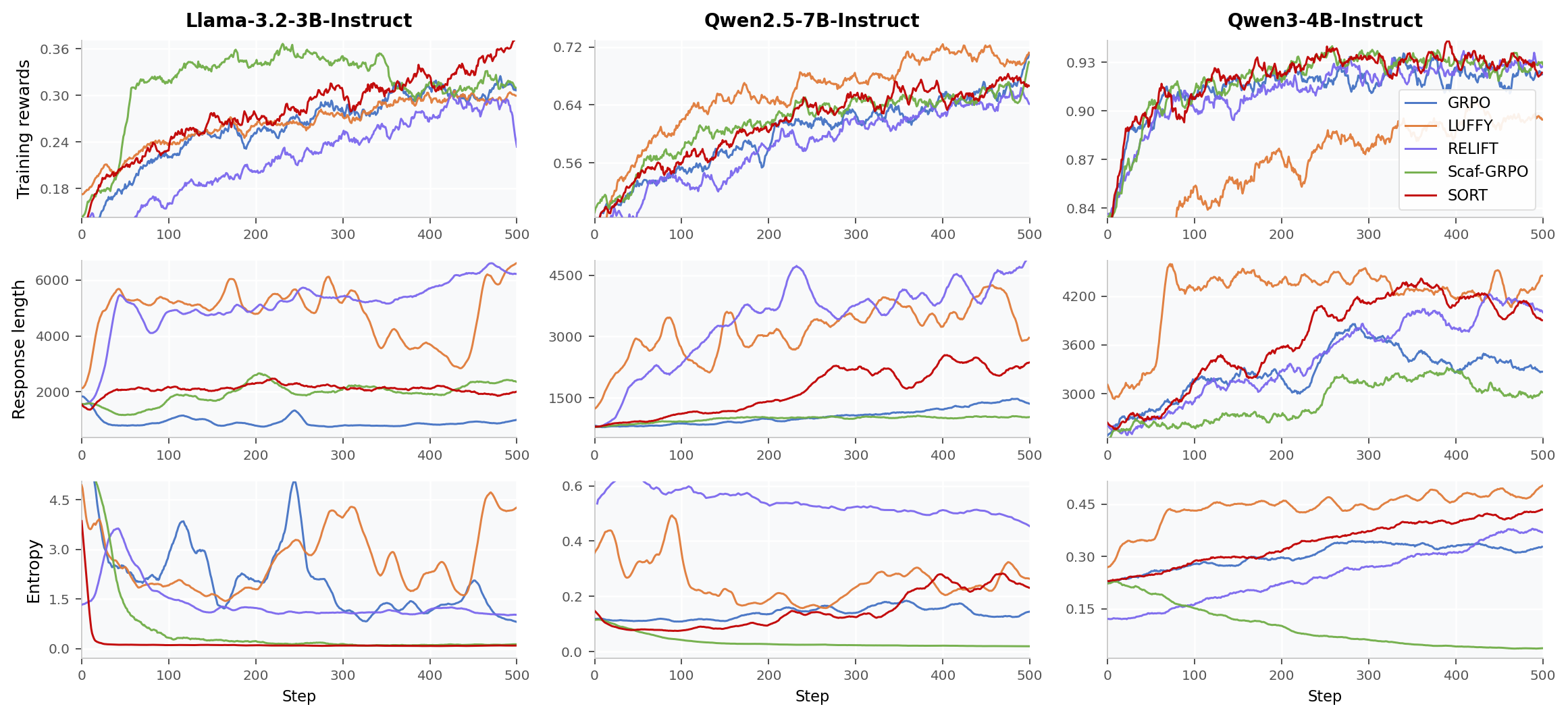}
\caption{Training dynamics across 500 optimization steps. Rows show reward, response length, and entropy; columns show the three backbones. SORT improves reward without the length expansion seen in LUFFY and ReLIFT.}
\label{fig:training_curves_3models}
\end{figure*}
\vspace{-0.8em}
Figure~\ref{fig:training_curves_3models} complements the final benchmark results by showing that higher reward is not merely a byproduct of longer generations. On Llama-3.2-3B-Instruct, SORT steadily improves to the highest reward while response length stays compact, whereas LUFFY and ReLIFT quickly expand length and Scaf-GRPO plateaus after sharp entropy collapse---the regime where all-wrong groups are most frequent, suggesting that selective reference repair guides the weak policy without broad trace imitation. On Qwen2.5-7B-Instruct, SORT catches up to LUFFY's early high reward at substantially lower length, indicating the plan-weighted update repairs hard prompts without verbosity cost. On Qwen3-4B-Instruct, where rewards near saturation, SORT tracks the best reward band, avoids LUFFY's long-response underperformance, and maintains moderate entropy unlike Scaf-GRPO. Across backbones, the curves support the central claim that SORT adds targeted pressure on hard cases instead of merely increasing length or uniformly imitating full demonstrations.

\subsection{Ablation Studies}
\label{sec:ablations}
\vspace{-0.5em}
We ablate SORT on Qwen2.5-7B-Instruct along two axes: the weighting exponent and the auxiliary-update variant.

\paragraph{Effect of $\beta$.}
Table~\ref{tab:beta_ablation} compares DFT-style student weighting ($\beta=0$), teacher-only weighting ($\beta=1$), and the geometric mean used by SORT ($\beta=1/2$). The midpoint gives the best in-distribution and out-of-distribution averages, improving from 41.6/48.3 for $\beta=0$ and 42.4/47.5 for $\beta=1$ to 43.5/50.1. This suggests that plan selectivity and student support are both needed.

\begin{table}[t]
\centering
\caption{Effect of the interpolation exponent $\beta$ on Qwen2.5-7B-Instruct.}
\label{tab:beta_ablation}
\setlength{\tabcolsep}{3.5pt}
\renewcommand{\arraystretch}{1.05}
\scriptsize
\resizebox{\columnwidth}{!}{%
\begin{tabular}{lccccccccccc}
\toprule
& \multicolumn{7}{c}{\textbf{In-distribution}}
& \multicolumn{4}{c}{\textbf{Out-of-distribution}} \\
\cmidrule(lr){2-8} \cmidrule(lr){9-12}
\textbf{Method}
& \textbf{AIME24/25} & \textbf{AMC23} & \textbf{MATH-500} & \textbf{Minerva} & \textbf{Olympiad} & \textbf{Avg.} & $\Delta$
& \textbf{GPQA} & \textbf{MMLU-Pro} & \textbf{Avg.} & $\Delta$ \\
\midrule

\textit{Qwen2.5-7B-Instruct}
& 13.8/6.7 & 53.4 & 75.7 & 38.1 & 39.2 & 37.8 & 0
& 37.1 & 56.4 & 46.7 & 0 \\

$\beta = 0$ 
& \textbf{17.7}/12.9 & 55.0 & 79.6 & 39.5 & 44.8 & 41.6 & +3.8
& 37.4 & 59.2 & 48.3 & +1.6 \\

$\beta = 1$ 
& 17.3/13.8 & 56.7 & 80.3 & 40.1 & 46.1 & 42.4 & +4.6
& 36.1 & 58.9 & 47.5 & +0.8 \\

$\beta = 1/2$ (SORT)
& 17.1/\textbf{15.4} & \textbf{60.5} & \textbf{81.0} & \textbf{39.5} & \textbf{47.3} & \textbf{43.5} & \textbf{+5.7}
& \textbf{39.6} & \textbf{60.5} & \textbf{50.1} & \textbf{+3.4} \\

\bottomrule
\end{tabular}%
}
\end{table}
\vspace{-0.8em}
\paragraph{Training variant.}
Table~\ref{tab:training_variant} separates buffered reference learning from plan-conditioned reweighting. Uniform SFT and plan-text supervision improve in-distribution accuracy to 42.2 and 42.4, but their out-of-distribution averages are 45.5 and 46.7, below the base model or tied with it. SORT reaches 43.5 in distribution and 50.1 out of distribution, indicating that the gain comes from token-selective geometric weighting rather than buffering alone.

\begin{table}[t]
\centering
\caption{Ablation on auxiliary update variants for Qwen2.5-7B-Instruct.}
\label{tab:training_variant}
\setlength{\tabcolsep}{3.5pt}
\renewcommand{\arraystretch}{1.05}
\scriptsize
\resizebox{\columnwidth}{!}{%
\begin{tabular}{lccccccccccc}
\toprule
& \multicolumn{7}{c}{\textbf{In-distribution}}
& \multicolumn{4}{c}{\textbf{Out-of-distribution}} \\
\cmidrule(lr){2-8} \cmidrule(lr){9-12}
\textbf{Method}
& \textbf{AIME24/25} & \textbf{AMC23} & \textbf{MATH-500} & \textbf{Minerva} & \textbf{Olympiad} & \textbf{Avg.} & $\Delta$
& \textbf{GPQA} & \textbf{MMLU-Pro} & \textbf{Avg.} & $\Delta$ \\
\midrule
\textit{Qwen2.5-7B-Instruct}
& 13.8/6.7 & 53.4 & 75.7 & 38.1 & 39.2 & 37.8 & 0
& 37.1 & 56.4 & 46.7 & 0 \\
SFT only        & 16.7/13.3 & 56.9  & \textbf{81.5} & 39.2 & 45.6 & 42.2 & +4.4 & 34.6  & 56.4 & 45.5 & -1.2 \\
SFT + plan       & 16.4/15.2 & 56.7  & 81.0 & 37.8 & \textbf{47.5} & 42.4 & +4.6 & 34.3  & 59.0 & 46.7 & 0.0 \\
SORT (Ours)     & \textbf{17.1}/\textbf{15.4} & \textbf{60.5} & 81.0 & \textbf{39.5} & 47.3 & \textbf{43.5} & \textbf{+5.7} & \textbf{39.6} & \textbf{60.5} & \textbf{50.1} & \textbf{+3.4} \\
\bottomrule
\end{tabular}%
}
\end{table}
\subsection{Token Salience Visualization}
\label{sec:token_salience_viz}
\begin{figure}[!t]
    \centering
    \includegraphics[width=0.9\linewidth]{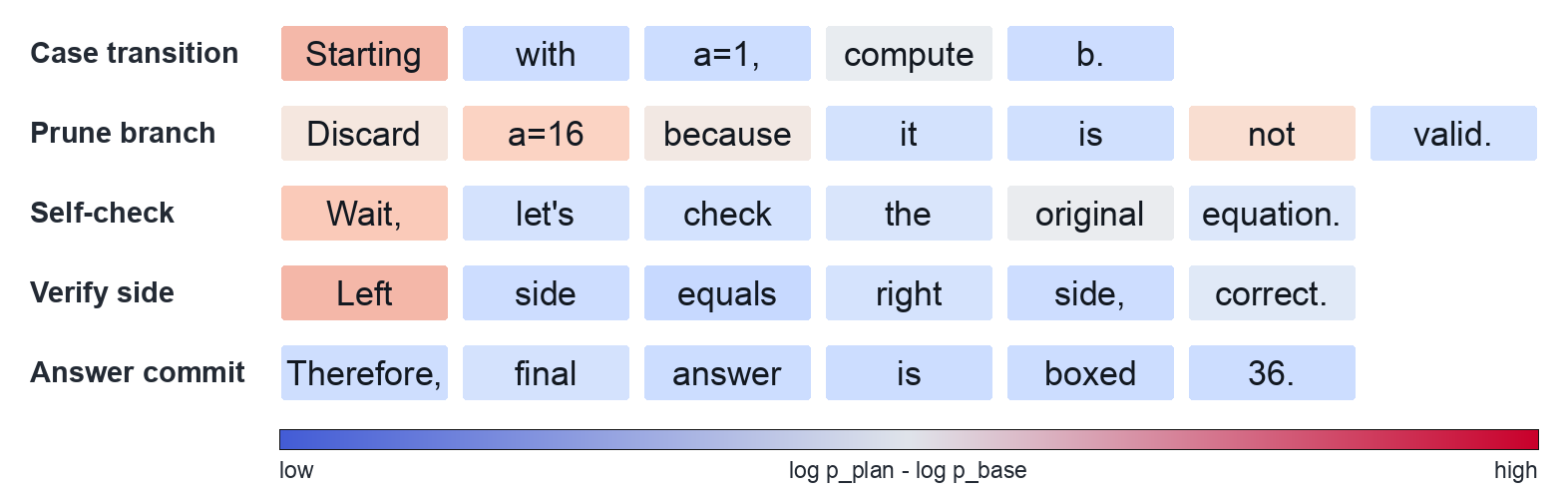}
\caption{Token-level salience on a zero-reward reference trace. Colors indicate the plan-conditioned log-ratio $\hat{\sigma}_{i,t}$, with high-salience regions concentrating on reasoning-control operations such as branching, verification, and self-checking.}
\label{fig:token_salience_excerpt}
\end{figure}
\vspace{-0.8em}
Figure~\ref{fig:token_salience_excerpt} illustrates the mechanism behind the token-selective update. Plan conditioning raises probability on tokens that steer the autoregressive reasoning trajectory, including case transitions, branch pruning, self-checks, and verification. Many local symbolic or answer-formatting tokens remain close to neutral because they are already predictable from the reference prefix. This supports the design of SORT as a policy-aware repair objective rather than uniform imitation of the full solution trace.



\section{Conclusion}

We introduced \textbf{Selective Off-Policy Reference Tuning with Plan Guidance (SORT)}, an auxiliary update for the zero-gradient failure mode of GRPO on all-wrong prompts. SORT buffers failed prompts, extracts a reference-derived solution plan, and uses same-model plan conditioning to identify reference tokens whose probabilities rise when the reasoning structure is supplied, yielding a DFT-style token-selective loss rather than uniform imitation. Our analysis connects this weight to an oracle structural-token objective and bounds the gap by the quality of the extracted plan. Across three instruction-tuned backbones and eight benchmarks, SORT consistently improves over GRPO and guidance-based baselines, with the largest gains on weaker models where all-wrong groups are most common.

\begin{ack}
Use unnumbered first level headings for the acknowledgments. All acknowledgments go at the end of the paper before the list of references. 

\end{ack}

\bibliographystyle{unsrtnat}
\bibliography{main}


\appendix
\section{Algorithm}\label{sec:alg}
\begin{algorithm}[h]
    \caption{GRPO with SORT Auxiliary Reference Updates}
    \label{alg:sort}
    \begin{algorithmic}[1]
    \REQUIRE Policy parameters $\theta$; rollout batch size $B_{\mathrm{rl}}$; rollouts per prompt $N$; auxiliary batch size $B_{\mathrm{aux}}$; empty buffer $\mathcal{B}_{0}$
    \FOR{$k = 1,\ldots,K$}
        \STATE Sample rollout batch $\mathcal{U}_k = \{q_i\}_{i=1}^{B_{\mathrm{rl}}}$
        \FOR{$i = 1,\ldots,B_{\mathrm{rl}}$}
            \STATE Sample $\{y_i^{(n)}\}_{n=1}^N \sim \pi_\theta(\cdot \mid q_i)$
            \STATE Compute verifiable rewards $\{r_i^{(n)}\}_{n=1}^N$
        \ENDFOR
        \STATE Compute GRPO advantages and update $\theta$ using $\mathcal{L}_{\mathrm{GRPO}}(\mathcal{U}_k)$
        \STATE $\mathcal{I}_{0}^{(k)} \gets \{\, i : r_i^{(n)} = 0,\; \forall n \in [N] \,\}$
        \STATE Enqueue $\{(q_i,y_i^\star): i\in \mathcal{I}_{0}^{(k)}\}$ into $\mathcal{B}_{0}$
        \IF{$|\mathcal{B}_{0}| \ge B_{\mathrm{aux}}$}
            \STATE Sample and remove $\mathcal{M}_k \sim \operatorname{SampleBatch}(\mathcal{B}_{0}, B_{\mathrm{aux}})$
            \FORALL{$(q_i,y_i^\star)\in \mathcal{M}_k$}
                \STATE $\hat{s}_i \gets \operatorname{Gen}_\theta(\textsc{PlanExtract}(y_i^\star))$
                \STATE Compute base probabilities $p^{\mathrm{base}}_{i,t}$ along $y_i^\star$
                \STATE Compute plan-conditioned probabilities $p^{\mathrm{plan}}_{i,t}$ along $y_i^\star$
                \STATE $\omega_{i,t} \gets
                \operatorname{sg}\!\left(
                \sqrt{
                p^{\mathrm{base}}_{i,t}
                p^{\mathrm{plan}}_{i,t}
                }
                \right)$ for all $t$
            \ENDFOR
            \STATE Update $\theta$ using $\mathcal{L}_{\mathrm{SORT}}(\mathcal{M}_k)$
        \ENDIF
    \ENDFOR
    \end{algorithmic}
    \end{algorithm}

\section{Limitations \& Broader Impacts}

\paragraph{Limitations.}
First, SORT assumes a verified reference solution is available for each zero-reward prompt. In domains where ground-truth solutions are expensive to annotate, the auxiliary update cannot be applied and the method reduces to vanilla GRPO on those prompts; our evaluation is also limited to mathematical reasoning and knowledge benchmarks, and transfer to other modalities such as code generation remains to be tested.
Second, like GRPO, SORT assumes verifiable binary outcome rewards. Extending the auxiliary repair mechanism to continuous, learned, or process-based reward signals is not addressed in this work.

\paragraph{Broader impacts.}
SORT improves the sample efficiency of RL-based reasoning training by converting previously wasted zero-reward prompts into structured learning signals. This has a positive practical impact: better reasoning ability from the same amount of training data, which reduces the compute and environmental cost of running large-scale RLVR pipelines. The method is general and could benefit any application where verifiable rewards and reference solutions are available, including formal theorem proving, code verification, and scientific reasoning.

On the potential negative side, SORT makes it easier to fine-tune language models toward stronger reasoning with less data, lowering the barrier for both beneficial and harmful uses. Improved mathematical reasoning could be leveraged for tasks such as automated vulnerability discovery or optimization of harmful processes. However, this is a general risk shared by virtually all LLM reasoning improvements, not one specific to SORT. The method does not introduce new failure modes beyond those already present in GRPO and reference-based fine-tuning. We see no reason why SORT would disproportionately increase misuse risk relative to other RLVR training techniques, and we encourage practitioners to apply standard responsible-release practices when deploying models trained with this method.

\section{Discussion}
\label{app:discussion}

\paragraph{Relation to self-distillation.}
Recent work suggests that self-distillation is a promising way to improve reasoning models without relying on a stronger external teacher
\citep{shenfeld2026selfdistillationenablescontinuallearning,hubotter2026reinforcement,zhao2026self}.
However, \citet{kim2026doesselfdistillationsometimesdegrade} show that self-distillation can mainly reinforce tokens that are already easy to predict from the previous context, while giving weaker signal to uncertain tokens. 
This is important for reasoning traces: many tokens, such as arithmetic continuations, algebraic simplifications, formatting tokens, or short connective phrases, are often predictable from the local prefix. 
Imitating these tokens is easy but provides little new reasoning signal.

Our method is designed to avoid over-emphasizing such routine tokens. 
If a token can already be predicted from the prefix, adding the plan changes its likelihood only slightly, so its salience remains small and it receives roughly ordinary weight. 
In contrast, tokens corresponding to uncertain reasoning choices, such as selecting a case split, introducing a key variable, applying a theorem, or moving to a new subgoal, can become much more predictable once the plan is provided. 
These tokens receive larger weights. 
Thus, rather than distilling the entire chain-of-thought uniformly, our objective keeps routine tokens near their normal weighting while upweighting the uncertain tokens that are made predictable by the reasoning plan.

\paragraph{Zero-reward prompts.}
This behavior is especially useful for zero-reward prompts, where all on-policy rollouts fail and GRPO provides no useful learning signal. 
A direct reference-imitation loss may still mostly train the model on easy parts of the solution, because those tokens dominate the trace and are locally predictable. 
Our method instead uses the reference solution to identify which tokens in the failed prompt correspond to uncertain reasoning steps. 
As shown in Section~\ref{app:full_token_salience}, tokens that are predictable from the prefix have near-zero salience, while uncertain tokens within the same solution step receive higher weights when plan guidance makes them predictable. 
This allows the model to learn from the reference without simply copying the full solution trace and lead to entropy collapse while training.

\paragraph{Implication.}
From this perspective, our method can be viewed as a selective form of self-distillation for reasoning. 
Rather than treating the model's or reference's entire chain-of-thought as equally useful supervision, we use plan-conditioned likelihood changes to separate routine continuation tokens from structurally informative tokens. 
This provides a simple mechanism for preserving the benefits of self-distillation while mitigating its tendency to over-emphasize easy tokens and under-emphasize uncertain reasoning steps.

\label{app:hard_examples}

To make the zero-reward regime in Section~\ref{sec:method} concrete, we show two examples on which Qwen2.5-7B-Instruct fails to produce any correct rollout in the sampled group. They illustrate prompts where the current policy misses the global reasoning structure needed to reach a verifiable answer.

\paragraph{Example 1: adaptive geometric search on a checkerboard.}
\emph{Qwen2.5-7B-Instruct outcome: zero-reward rollout group; ground-truth answer: 2.}

\begin{quote}
Bakayev E.V. Petya drew a square on the plane, divided it into 64 identical smaller squares, and colored them in a checkerboard pattern with black and white colors. After that, he thought of a point strictly inside one of these smaller squares. Vasya can draw any closed broken line without self-intersections on the plane and get an answer to the question of whether the guessed point is strictly inside the broken line or not. What is the minimum number of such questions Vasya needs to determine the color of the guessed point---white or black?
\end{quote}

This prompt is hard for sparse-reward RL because a correct answer depends on designing an adaptive querying strategy and proving optimality. A partially correct rollout can identify the board structure or propose a query, but the verifier only rewards the final minimum and its justification. Below is a representative Qwen2.5-7B-Instruct rollout that answers $1$ (correct is $2$): the model reduces the problem to parity of grid coordinates but overlooks that one question alone cannot distinguish both colors from the exterior region.

\smallskip
\noindent\emph{Model output (incorrect).}
\begin{quote}\small
First, note that the square is divided into $8 \times 8 = 64$ smaller squares in a checkerboard pattern. The color of a square depends on the parity of the sum of its row and column indices. Specifically, if $(i+j) \bmod 2 = 0$ the square is black, otherwise it is white. The key observation is that if we can determine the parity of $(i+j)$ for the square containing the point, we know the color.

Vasya can draw a closed broken line that separates squares with even $i+j$ from those with odd $i+j$. For instance, a line running from $(1,1)$ to $(8,8)$ partitions the board by parity. By asking whether the point is inside this region, Vasya learns the parity and therefore the color.

Thus, the minimum number of questions is $\boxed{1}$.
\end{quote}

The rollout misses that one query cannot reliably separate the two parity classes while also handling the exterior region, so the answer requires at least two adaptive questions. In the observed Qwen2.5-7B-Instruct rollouts, this global query strategy is not recovered, so the group receives no positive reward.

\paragraph{Example 2: worst-case color separation by lines.}
\emph{Qwen2.5-7B-Instruct outcome: zero-reward rollout group; ground-truth answer: 2013.}

\begin{quote}
In the plane, 2013 red points and 2014 blue points are marked so that no three of the marked points are collinear. One needs to draw \(k\) lines not passing through the marked points and dividing the plane into several regions. The goal is to do it in such a way that no region contains points of both colors. Find the minimal value of \(k\) such that the goal is attainable for every possible configuration of 4027 points. (Australia)
\end{quote}

This prompt combines extremal geometry with a worst-case quantifier over all configurations. A successful solution must provide both a construction showing sufficiency and an adversarial configuration showing necessity. Below is a representative Qwen2.5-7B-Instruct rollout that answers $2015$ (correct is $2013$): the model reasons about separating convex hulls but loses the tight count on the required number of lines.

\smallskip
\noindent\emph{Model output (incorrect).}
\begin{quote}\small
First, consider the convex hull of all the points. Since no three points are collinear, the convex hull is a polygon with at most 4027 vertices (the maximum when all points lie on the hull).

The key insight is that we can separate the convex hull into two regions by drawing lines such that one region contains only red points and the other only blue points. This is done by drawing lines that pass between the convex hull and the rest of the points.

To ensure no region contains points of both colors, we need to separate the convex hull from the remaining points. Each side of the convex hull can be ``cut'' by one line, requiring at most $\text{vertices} - 2 = 4027 - 2 = 4025$ lines, plus one additional line to guarantee complete separation, yielding $4026$ lines. However, a tighter analysis using the convex hull of the minority color (2013 red points) gives $2014$ lines to cut around them, plus one final separator, for a total of $2015$.

Thus, the minimal value of $k$ is $\boxed{2015}$.
\end{quote}

The rollout uses convex-hull reasoning in the right spirit but mishandles the counting: the correct upper-bound construction uses $2013$ lines, each separating one red point from all blue points, and the matching lower bound exhibits a configuration where fewer than $2013$ lines cannot suffice. In the observed Qwen2.5-7B-Instruct rollouts, local geometric intuitions lead to plausible but unverifiable partial arguments, and the full upper- and lower-bound structure needed for the answer is absent.

\section{Example Ground-Truth Solution Plan}
\label{app:ground_truth_plan_example}

For the logarithmic equation
\[
\log _{\sqrt{a}} \frac{\sqrt{2 a - x}}{a}-\log _{1 / a} x=0,
\]
the plan extraction prompt produces a thinking process that captures the key reasoning decisions, self-checks, and intermediate realizations:

\begin{quote}
First, I need to handle the domain of the logarithms. For the first log to be defined, the base $\sqrt{a}>0$, $\sqrt{a}\ne 1$, so $a>0$, $a\ne 1$. The argument $\frac{\sqrt{2a-x}}{a}>0$ — since $a>0$, this means $\sqrt{2a-x}>0$, so $2a-x>0$, giving $x<2a$. For $\log_{1/a}x$, the base $1/a>0$ and $1/a\ne 1$ are already covered by $a>0$, $a\ne 1$, and the argument gives $x>0$. So domain: $a>0$, $a\ne 1$, $0<x<2a$.

Now to solve the equation. The bases $\sqrt{a}$ and $1/a$ are inconvenient — let me rewrite them as powers of $a$. $\sqrt{a}=a^{1/2}$, $1/a=a^{-1}$.

Using the change of base formula: $\log_{a^{1/2}} u = \frac{\log_a u}{\log_a a^{1/2}} = \frac{\log_a u}{1/2} = 2\log_a u$. So the first term is $2\log_a(\frac{\sqrt{2a-x}}{a})$.

Wait — actually, let me be more careful. $\log_{\sqrt{a}} \frac{\sqrt{2a-x}}{a}$. Using base change: $\log_{a^{1/2}} u = 2\log_a u$. With $u = \frac{\sqrt{2a-x}}{a} = \frac{\sqrt{2a-x}}{a}$, this is $\log_a \frac{\sqrt{2a-x}}{a} = \log_a\sqrt{2a-x} - \log_a a = \frac{1}{2}\log_a(2a-x) - 1$.

Hmm — so doubled: $2(\frac{1}{2}\log_a(2a-x) - 1) = \log_a(2a-x) - 2$.

Now the second term: $\log_{1/a} x = \log_{a^{-1}} x = \frac{\log_a x}{\log_a a^{-1}} = \frac{\log_a x}{-1} = -\log_a x$.

The equation becomes: $(\log_a(2a-x)-2) - (-\log_a x) = 0$, so $\log_a(2a-x) - 2 + \log_a x = 0$.

Combine: $\log_a[x(2a-x)] = 2$.

Exponentiate: $x(2a-x) = a^2$, so $2ax - x^2 = a^2$, giving $x^2 - 2ax + a^2 = 0$.

This factors as $(x-a)^2 = 0$, so $x = a$.

Check against the domain: $a>0$, $a\ne 1$, and $0 < a < 2a$ which holds for all $a>0$. So $x=a$ is valid. Final answer: $x=a$.
\end{quote}

The extracted plan is more conversational than a polished solution but exposes the global decisions: domain analysis, base conversion with a self-check, logarithm combination, quadratic reduction, and final domain verification. These are precisely the tokens whose probabilities should rise under plan conditioning.

\section{Qualitative Plan Extraction Example}
\label{app:plan_extraction_example}

We show a representative example of the plan-extraction step used before plan-conditioned scoring. The prompt asks for a two-digit natural number \(\overline{ab}\) satisfying
\[
\sqrt{a+\sqrt{\overline{ab}}}=a.
\]

The full reference solution contains both high-level structural reasoning and lower-level surface realization, including intermediate checks, self-corrections, and verification steps:

\begin{quote}\small
Okay, so I need to find a two-digit natural number \(\overline{ab}\) such that
\[
\sqrt{a+\sqrt{\overline{ab}}}=a.
\]
Since \(\overline{ab}\) is a two-digit number, \(a\) is the tens digit and \(b\) the units digit, so
\[
\overline{ab}=10a+b,
\]
with \(a\in\{1,\ldots,9\}\) and \(b\in\{0,\ldots,9\}\).

Substitute into the equation:
\[
\sqrt{a+\sqrt{10a+b}}=a.
\]

Square both sides:
\[
a+\sqrt{10a+b}=a^2.
\]

Subtract \(a\):
\[
\sqrt{10a+b}=a^2-a.
\]

Square again:
\[
10a+b=(a^2-a)^2.
\]

Expand:
\[
(a^2-a)^2=a^4-2a^3+a^2,
\]
so
\[
10a+b=a^4-2a^3+a^2.
\]

Rearranging,
\[
b=a^4-2a^3+a^2-10a.
\]

Since \(b\) must be a digit from \(0\) to \(9\), test values of \(a\).

For \(a=1\),
\[
b=1-2+1-10=-10,
\]
which is invalid.

For \(a=2\),
\[
b=16-16+4-20=-16,
\]
also invalid.

For \(a=3\),
\[
b=81-54+9-30.
\]
Wait, \(81-54=27\), then \(27+9=36\), and \(36-30=6\). So \(b=6\), which is a valid digit.

Thus the candidate number is \(36\). Check:
\[
\sqrt{3+\sqrt{36}}
=
\sqrt{3+6}
=
\sqrt{9}
=
3.
\]

Now check whether larger values of \(a\) can work.

For \(a=4\),
\[
b=256-128+16-40=104,
\]
which is too large to be a digit.

For \(a=5\),
\[
b=625-250+25-50=350,
\]
also impossible. Larger \(a\) values only make \(b\) larger.

Therefore the unique solution is
\[
\boxed{36}.
\]
\end{quote}

The extracted thinking process preserves the key mathematical structure and genuine reasoning transitions while removing most repetitive exposition and stylistic filler:

\begin{quote}\small
Find a two-digit number \(\overline{ab}\) such that
\[
\sqrt{a+\sqrt{\overline{ab}}}=a.
\]

Since \(\overline{ab}\) is two-digit,
\[
\overline{ab}=10a+b,
\]
so
\[
\sqrt{a+\sqrt{10a+b}}=a.
\]

Square both sides:
\[
a+\sqrt{10a+b}=a^2,
\qquad
\sqrt{10a+b}=a^2-a.
\]

Square again:
\[
10a+b=(a^2-a)^2.
\]

Expand:
\[
(a^2-a)^2=a^4-2a^3+a^2,
\]
hence
\[
b=a^4-2a^3+a^2-10a.
\]

Since \(b\) must be a digit, test values of \(a\).

For \(a=1\),
\[
b=-10,
\]
invalid.

For \(a=2\),
\[
b=-16,
\]
also invalid.

For \(a=3\),
\[
b=81-54+9-30.
\]
Wait, \(81-54=27\), then \(27+9=36\), and \(36-30=6\). So \(b=6\).

Check:
\[
\sqrt{3+\sqrt{36}}
=
\sqrt{3+6}
=
\sqrt{9}
=
3.
\]

For \(a\ge4\), \(b\) becomes too large to be a digit. Therefore,
\[
\overline{ab}=36.
\]
\end{quote}

This example illustrates that plan extraction retains the essential solution trajectory, intermediate calculations, and moments of genuine correction, while discarding verbose explanatory narration that does not materially affect the reasoning structure.

\section{Prompt Templates for Plan Extraction and Teacher Conditioning}
\label{app:prompt_templates}

We follow the prompting setup of \citet{luffy} and use the same system prompts for SORT and all baseline methods.
Specifically, for Qwen backbones we adopt the longer LUFFY reasoning prompt, while for the LLaMA-3.2-3B backbone we use a simplified step-by-step prompt, as the model does not reliably follow the longer instruction.

\paragraph{System prompt (Qwen backbones).}
\begin{quote}
\ttfamily
Your task is to follow a systematic, thorough reasoning process before providing the final solution. This involves analyzing, summarizing, exploring, reassessing, and refining your thought process through multiple iterations. Structure your response into two sections: Thought and Solution. In the Thought section, present your reasoning using the format: <think>  {thoughts}  </think>. Each thought should include detailed analysis, brainstorming, verification, and refinement of ideas. After </think>, in the Solution section, provide the final, logical, and accurate answer, clearly derived from the exploration in the Thought section. If applicable, include the answer in boxed{} for closed-form results like multiple choices or mathematical solutions.
\end{quote}

\paragraph{System prompt (LLaMA-3.2-3B-Instruct).}
\begin{quote}
\ttfamily
Please reason step by step, and put your final answer within \textbackslash boxed\{\}.
\end{quote}

\paragraph{Plan extraction prompt.}
\begin{quote}
\ttfamily
You are given a detailed reasoning process.
Extract only the key thinking steps verbatim, including intermediate calculations, dead ends, self-corrections, and the intuition behind each key realization.
Include moments of uncertainty only when they naturally arise.
Do not add decorative explanations.\\
\\
Reasoning: \{solution\}\\
\\
Thinking process:
\end{quote}

\paragraph{Teacher prompt.}
\begin{quote}
\ttfamily
\{problem\}\\
Here is the thinking process to solve this problem:\\
\{blueprint\}
\end{quote}

\section{All-Wrong Sample Fraction During SORT Training}
\label{app:all_wrong_fraction}

We additionally track the number of \textit{all-wrong samples} during SORT training, i.e., rollout groups where every sampled trajectory receives zero reward under GRPO. Concretely, for a prompt $q_i$, this corresponds to

\[
r_i^{(1)} = \cdots = r_i^{(N)} = 0,
\]

which is exactly the zero-reward regime targeted by SORT in Section~3.1. Figure~\ref{fig:wrong_case_fraction} shows the evolution of the number of all-wrong prompts across training for Qwen3-4B-Instruct, Qwen2.5-7B-Instruct, and Llama-3.2-3B-Instruct.

\begin{figure}[t]
    \centering
    \includegraphics[width=\linewidth]{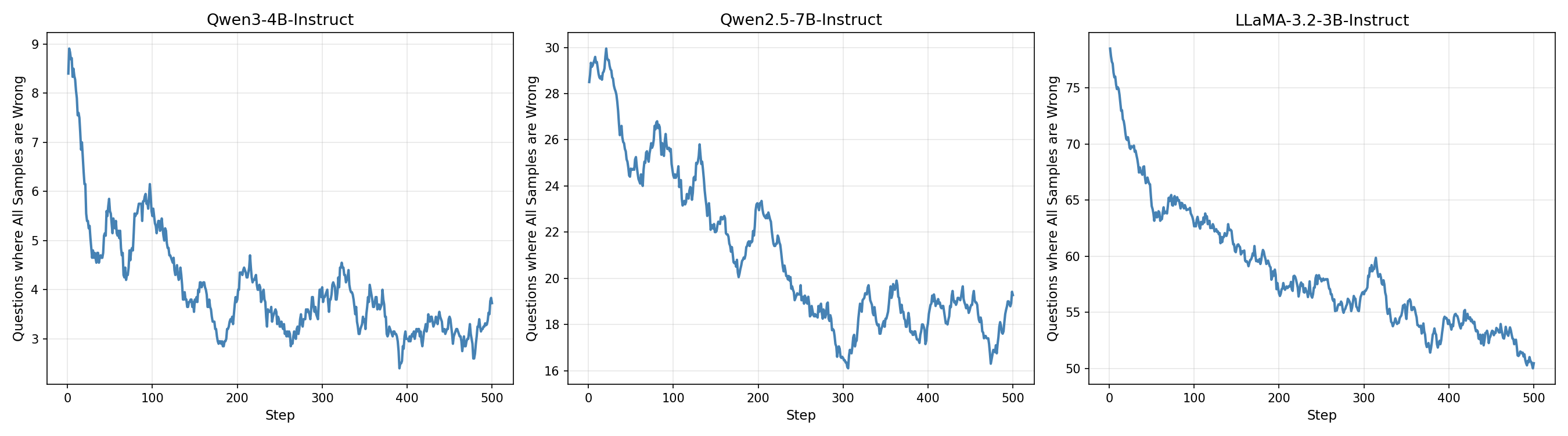}
    \caption{
    Number of all-wrong rollout groups during SORT training, where every sampled trajectory receives zero reward. Lower values indicate fewer zero-reward prompts.
    }
    \label{fig:wrong_case_fraction}
\end{figure}

Several trends support the motivation of SORT. Weaker models exhibit substantially more all-wrong rollout groups, with Llama-3.2-3B-Instruct starting near 80 zero-reward prompts per batch, Qwen2.5-7B-Instruct around 30, and Qwen3-4B-Instruct below 10. Across all backbones, the number of all-wrong groups decreases steadily during training, indicating that RL optimization gradually expands the set of prompts for which at least one successful trajectory can be sampled. The largest reductions occur for weaker models, matching the observation in Table~1 that SORT provides the strongest gains in low-capability regimes. At the same time, residual zero-reward prompts persist even for stronger models, suggesting that SORT-style hard-prompt repair remains useful throughout training rather than only in the early stages.

This extracted plan is not used as a target sequence. It is only prepended to the teacher context when scoring the original reference tokens. As a result, tokens corresponding to structural operations--substituting \(\overline{ab}=10a+b\), squaring twice, deriving the digit constraint on \(b\), and verifying \(36\)--become more predictable under plan conditioning. Routine phrasing and repeated arithmetic receive little additional support, which is the behavior used by SORT to distinguish plan-informative tokens from surface tokens. In terms of Assumption~\ref{assump:stability}, this is a favorable case: the extracted plan preserves the reasoning state needed to predict the reference path, so its induced hidden states should be close to those under the ideal plan. If the extractor dropped the digit constraint or the verification step, the corresponding \(\gamma\) would be larger and the regret bound in Theorem~\ref{thm:regret_bound} would become less informative.

\section{Runtime and Compute Overhead}
\label{app:runtime_overhead}

SORT adds only a modest overhead over vanilla GRPO. 
For each buffered zero-reward prompt, SORT performs one plan-generation step and one additional forward pass to compute the plan-conditioned token likelihoods along the reference trajectory. 
As a result, the runtime remains comparable to existing guidance-based baselines, as shown in Table~\ref{tab:runtime_overhead}.

\begin{table}[h]
\centering
\caption{Training time for Qwen2.5-7B-Instruct. Runtime is reported relative to vanilla GRPO.}
\label{tab:runtime_overhead}
\setlength{\tabcolsep}{4pt}
\renewcommand{\arraystretch}{1.08}
\small
\resizebox{0.8\linewidth}{!}{%
\begin{tabular}{lccccc}
\toprule
\textbf{Method}
& \textbf{GRPO}
& \textbf{LUFFY}
& \textbf{ReLIFT}
& \textbf{Scaf-GRPO}
& \textbf{SORT} \\
\midrule
\textbf{Training time}
& 1.0x (19.0h)
& 1.2x
& 1.95x
& 1.5x
& 1.55x \\
\bottomrule
\end{tabular}%
}
\end{table}

\section{SFT Loss Dynamics}
\label{app:sft_loss_dynamics}

Figure~\ref{fig:sft_loss_comparison} reports the SFT-path loss over training for the auxiliary reference-learning variants in Section~\ref{sec:ablations}. Uniform SFT variants keep a substantially higher loss because they place learning pressure on the whole reference trace, including routine algebra, formatting, and already-predictable tokens. The $\beta=0$ setting, corresponding to DFT-style student weighting, yields the smallest loss by downweighting reference tokens according to base-model support. SORT ($\beta=1/2$) stays close to this low-loss DFT regime but is intentionally above it: plan conditioning selectively boosts structurally important tokens, adding targeted learning pressure without returning to full-trace imitation. This diagnostic complements Tables~\ref{tab:beta_ablation} and~\ref{tab:training_variant}: it shows how the reference objective behaves during optimization, whereas the tables report downstream evaluation accuracy.

\begin{figure}[H]
    \centering
    \includegraphics[width=\linewidth]{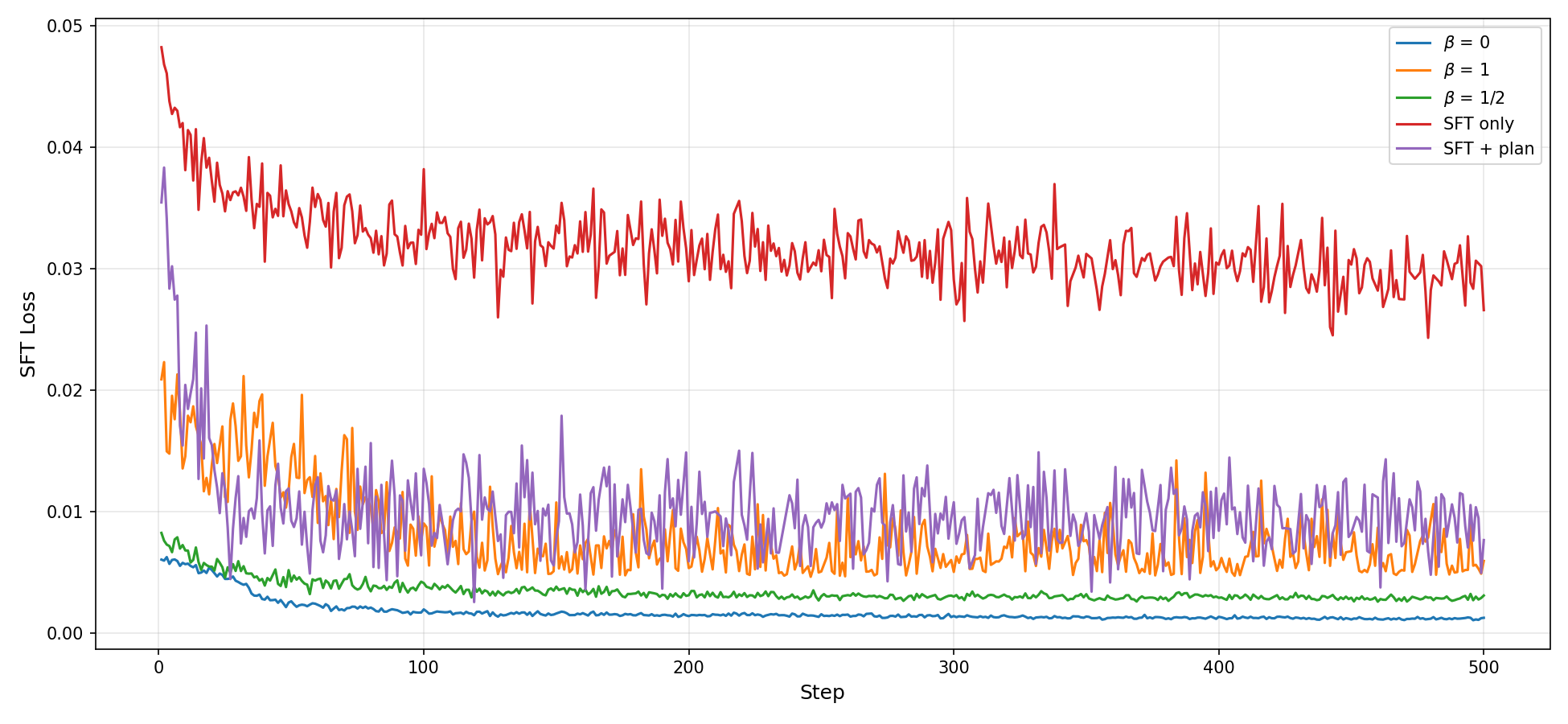}
    \caption{SFT-path loss trajectories across 500 training steps for the auxiliary reference-learning variants on Qwen2.5-7B-Instruct. Uniform SFT baselines maintain high loss because they learn from all reference tokens. DFT-style weighting ($\beta=0$) gives the lowest loss, while SORT ($\beta=1/2$) remains near DFT but higher due to selective boosting of plan-informative tokens.}
    \label{fig:sft_loss_comparison}
\end{figure}

\section{Full Token Salience Visualization}
\label{app:full_token_salience}

Figure~\ref{fig:full_token_salience_1}--\ref{fig:full_token_salience_4} show the full token-level visualization corresponding to the excerpt in Figure~\ref{fig:token_salience_excerpt}. Each cell is one token from the reference trace, colored by the plan-conditioned log-ratio
\[
\hat{\sigma}_{i,t}
=
\log p^{\mathrm{plan}}_{i,t}
-
\log p^{\mathrm{base}}_{i,t}.
\]
Redder tokens become more predictable under plan conditioning, while bluer tokens become less predictable.

\begin{figure}[h]
    \centering
    \includegraphics[width=\linewidth]{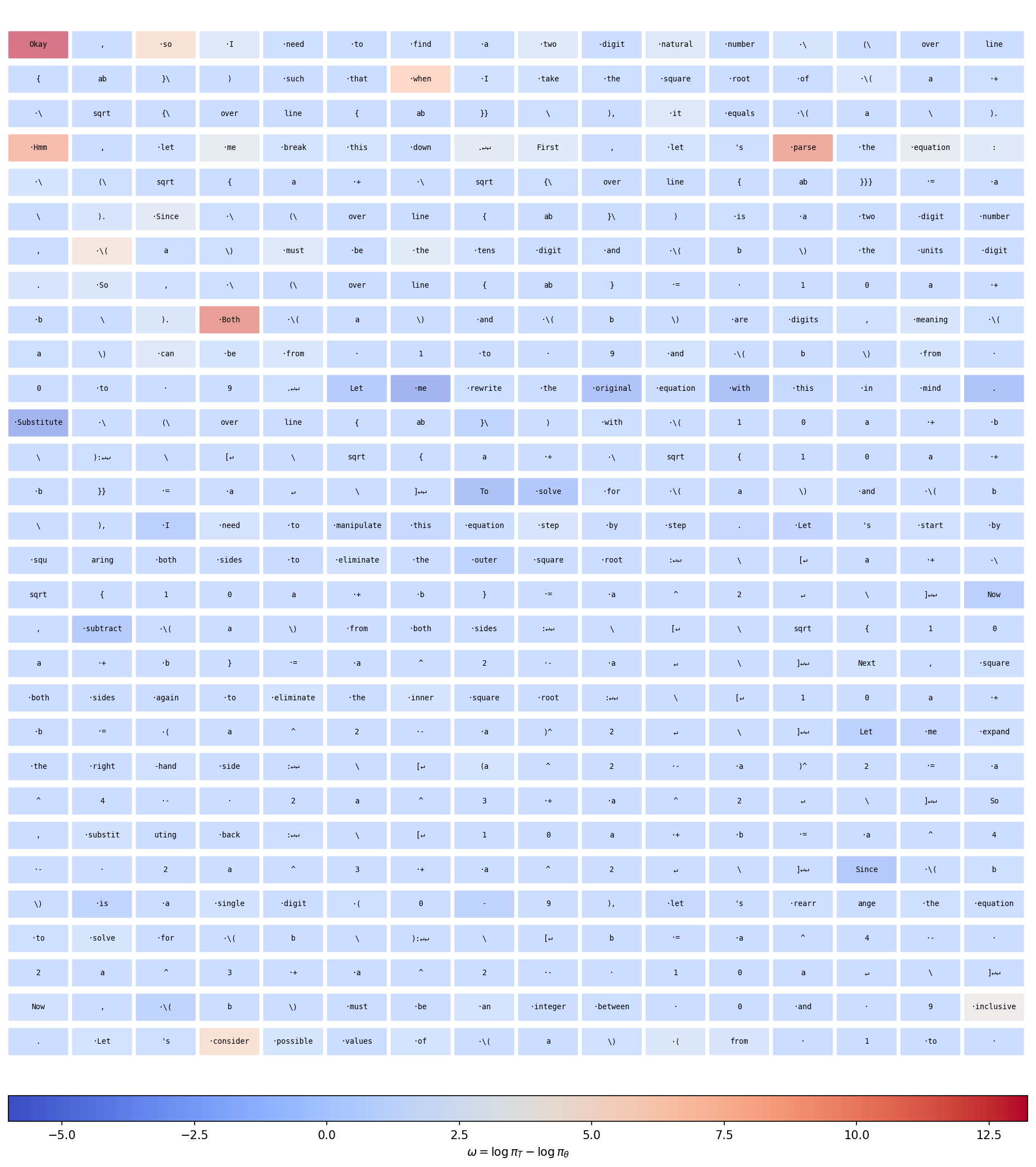}
    \caption{Full token salience visualization, part 1.}
    \label{fig:full_token_salience_1}
\end{figure}

\begin{figure}[p]
    \centering
    \includegraphics[width=\linewidth]{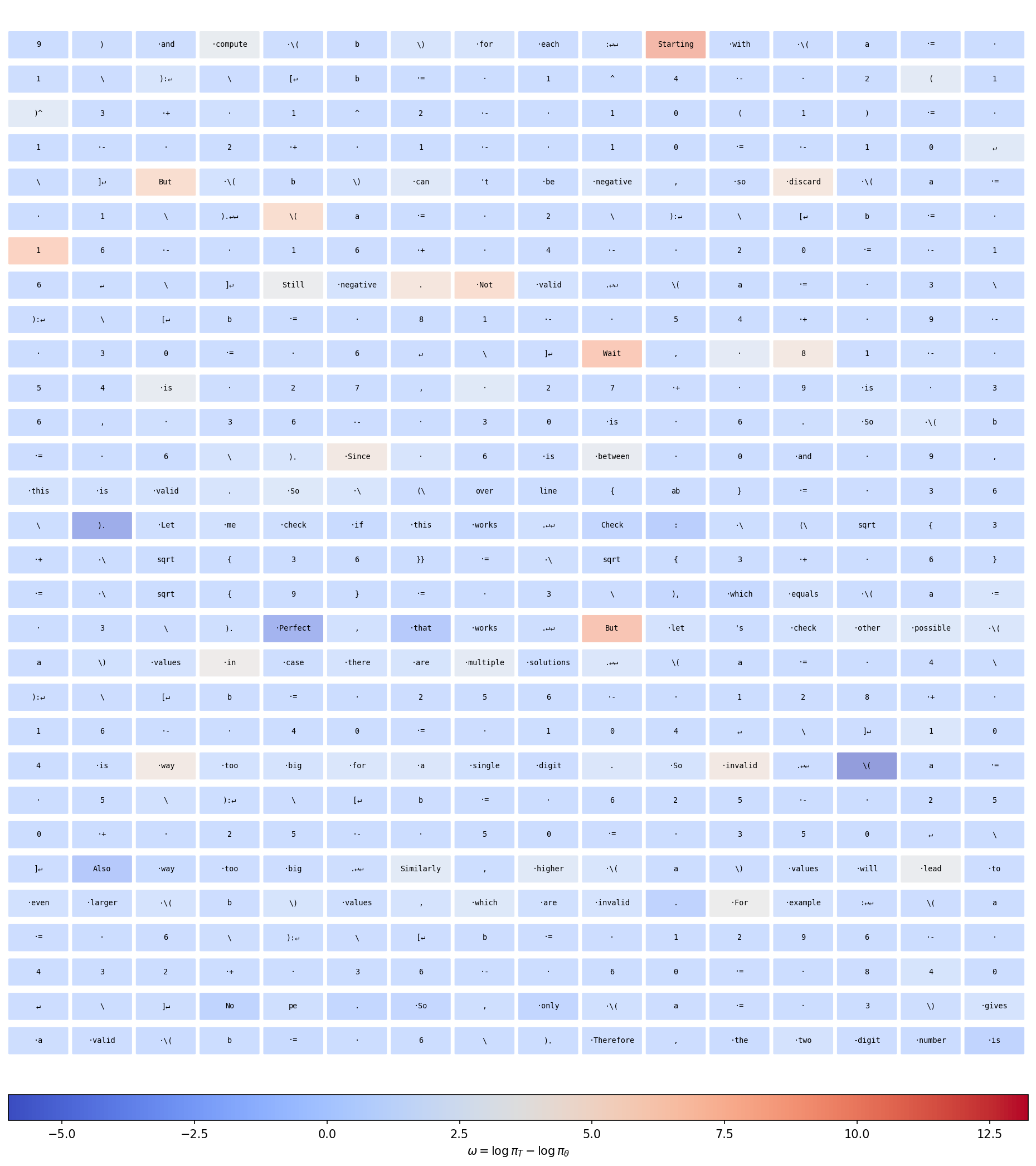}
    \caption{Full token salience visualization, part 2.}
    \label{fig:full_token_salience_2}
\end{figure}

\begin{figure}[p]
    \centering
    \includegraphics[width=\linewidth]{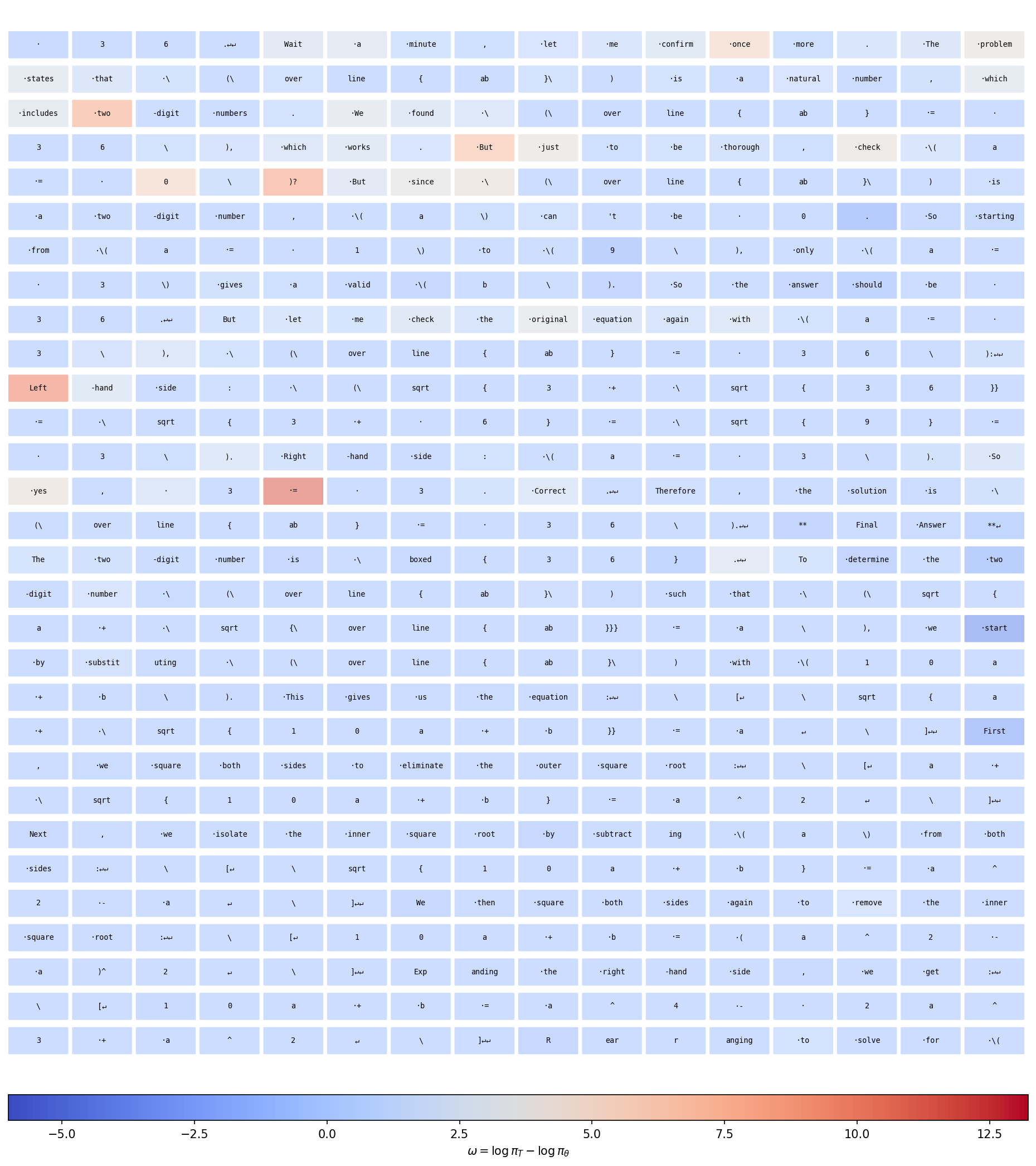}
    \caption{Full token salience visualization, part 3.}
    \label{fig:full_token_salience_3}
\end{figure}

\begin{figure}[p]
    \centering
    \includegraphics[width=\linewidth]{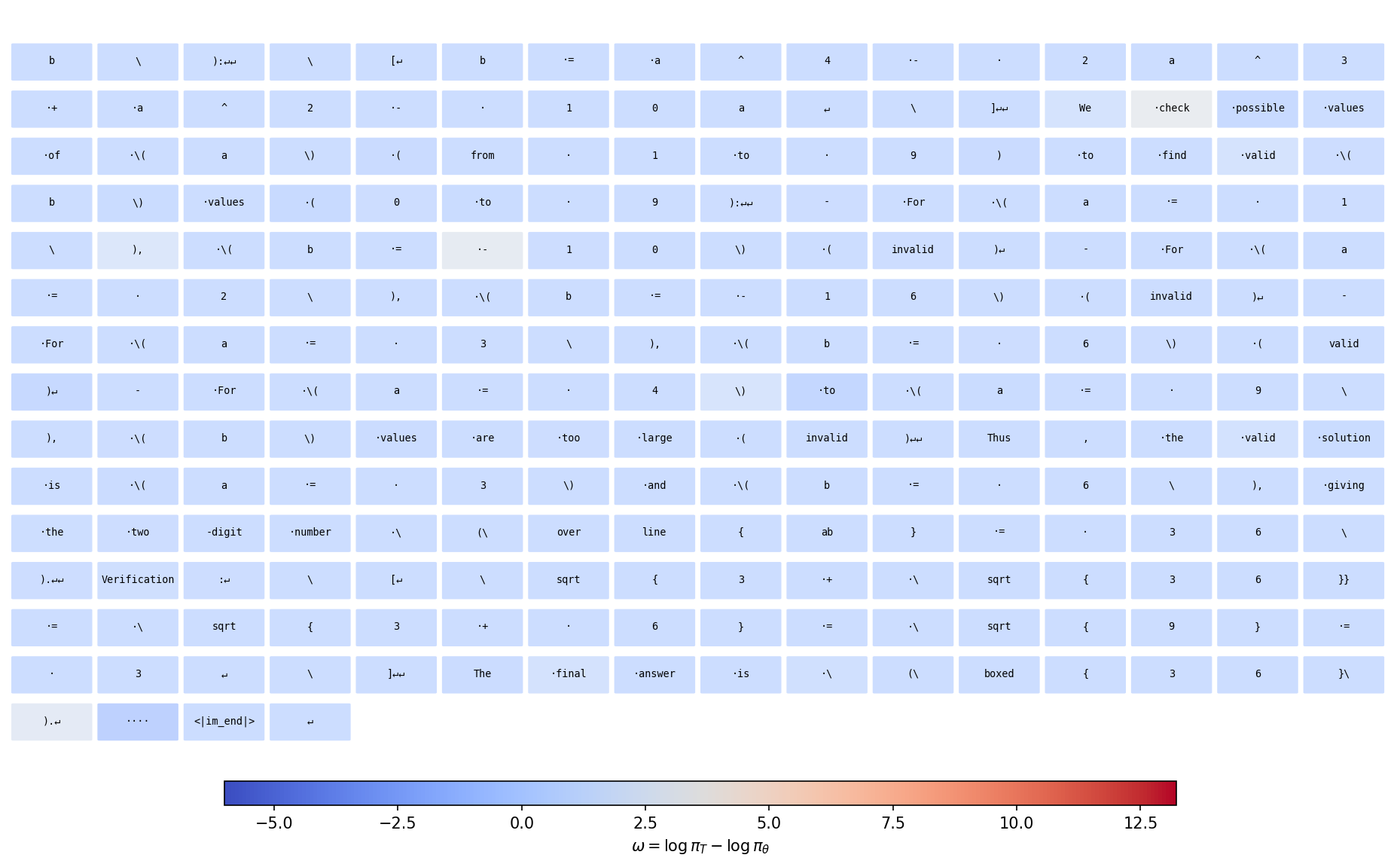}
    \caption{Full token salience visualization, part 4.}
    \label{fig:full_token_salience_4}
\end{figure}

\clearpage


\section{Proofs for Section~\ref{sec:theoretical_analysis}}
\label{app:regret_bound}

Definitions~\ref{def:ground_truth_solution_plan}--\ref{def:plan_salience}, the oracle weight~\eqref{eq:oracle_weight}, and Assumptions~\ref{assump:stability}--\ref{assump:bounded_score} are given in the main text.

\begin{proof}[Proof of Lemma~\ref{lem:plan_to_token}]
Fix any token position $t$. By the $L$-Lipschitz assumption on $h \mapsto \log\pi_\theta(\cdot \mid h)$,
\[
\bigl|\log\pi_\theta(y_{i,t}^\star \mid \hat{s}_i, q_i, y_{i,<t}^\star) - \log\pi_\theta(y_{i,t}^\star \mid S_i, q_i, y_{i,<t}^\star)\bigr|
\le L\,\|h_\theta(q_i, y_{i,<t}^\star, \hat{s}_i) - h_\theta(q_i, y_{i,<t}^\star, S_i)\|.
\]
Assumption~\ref{assump:stability} bounds the right-hand side by $L\gamma$ pointwise for all $t$.
\end{proof}

\begin{proof}[Proof of Theorem~\ref{thm:regret_bound}]
Let $p_{i,t} = p^{\mathrm{base}}_{i,t}$ and $q_{i,t} = p^{\mathrm{plan}}_{i,t}$ as defined in~\eqref{eq:base_plan_prob}, and let $o_{i,t} = \pi_\theta(y_{i,t}^\star \mid S_i, q_i, y_{i,<t}^\star)$ be the oracle plan-conditioned probability. All lie in $(0,1]$, so their logs are non-positive. From~\eqref{eq:sort_weight_family} and~\eqref{eq:oracle_weight},
\[
\omega_{\beta,i,t} = p_{i,t}^{1-\beta} q_{i,t}^{\beta}, \qquad
\omega_{\beta,i,t}^\star = p_{i,t}^{1-\beta} o_{i,t}^{\beta}.
\]
Write $\ell_q = \log q_{i,t}$, $\ell_o = \log o_{i,t}$. Then
\begin{align}
\bigl|\omega_{\beta,i,t} - \omega_{\beta,i,t}^\star\bigr|
&= p_{i,t}^{1-\beta} \cdot \bigl|e^{\beta\ell_q} - e^{\beta\ell_o}\bigr|
\le p_{i,t}^{1-\beta} \cdot \beta\,|\ell_q - \ell_o|
\le \beta L\gamma\,p_{i,t}^{1-\beta},
\end{align}
where the first inequality uses the mean value theorem on $f(x)=e^{\beta x}$: $|f(a)-f(b)|=\beta e^{\beta\xi}|a-b|\le \beta|a-b|$ since $\beta\in[0,1]$ and $\xi\le\max(a,b)\le0$. The second inequality applies Lemma~\ref{lem:plan_to_token} pointwise, noting that $|\ell_q - \ell_o| = |\log p^{\mathrm{plan}}_{i,t} - \log \pi_\theta(y_{i,t}^\star \mid S_i, q_i, y_{i,<t}^\star)| \le L\gamma$. For $\beta=\tfrac{1}{2}$ this gives $|\omega_{i,t}-\omega_{i,t}^\star|\le (L\gamma/2)\sqrt{p_{i,t}}$.

For the gradient: $\nabla_\theta \mathcal{L}_{\mathrm{SORT},\beta} = -\frac{1}{|\mathcal{S}|}\sum_{(i,t)} \omega_{\beta,i,t} \nabla_\theta \log p_{i,t}$, and identically for $\mathcal{L}^\star_\beta$ with $\omega_{\beta,i,t}^\star$. Hence
\begin{align}
\bigl\|\nabla_\theta \mathcal{L}_{\mathrm{SORT},\beta} - \nabla_\theta \mathcal{L}^\star_\beta\bigr\|
&\le \frac{1}{|\mathcal{S}|}\sum_{(i,t)} |\omega_{\beta,i,t} - \omega_{\beta,i,t}^\star| \cdot \|\nabla_\theta \log p_{i,t}\| \\
&\le \beta L\gamma \cdot \frac{1}{|\mathcal{S}|}\sum_{(i,t)} p_{i,t}^{1-\beta} \, \|\nabla_\theta \log p_{i,t}\|
\le \beta L\gamma G,
\end{align}
where the last step uses $p_{i,t}^{1-\beta} \le 1$ and $\|\nabla_\theta \log p_{i,t}\| \le G$ from Assumption~\ref{assump:bounded_score}.
\end{proof}



\end{document}